\documentclass{article}

\usepackage{algorithm,algpseudocode}
\usepackage{mathtools,dsfont}
\usepackage{psfrag,bbm,graphicx}
\usepackage{comment,balance}
\usepackage{epstopdf}
\usepackage{epsfig}
\usepackage{amsmath}
\usepackage{subcaption}
\usepackage{booktabs}
\usepackage{wrapfig}
\usepackage{multicol}
\usepackage{multirow}
\usepackage{color}
\usepackage{array}
\usepackage{amssymb}
\usepackage{amsthm,bm}
\usepackage{etoolbox}
\usepackage[hidelinks]{hyperref}
\usepackage{enumitem}
\usepackage[dblblindworkshop, final,nonatbib]{neurips_2025}
\usepackage{defs_neurips}

\usepackage{blindtext}

\workshoptitle{\\ Workshop on Algorithmic Collective Action}

\usepackage[utf8]{inputenc} 
\usepackage[T1]{fontenc}    
\usepackage{hyperref}       
\usepackage{url}            
\usepackage{booktabs}       
\usepackage{amsfonts}       
\usepackage{nicefrac}       
\usepackage{microtype}      
\usepackage{xcolor}         

\title{Observation-Free Attacks on Online Learning to Rank}

\author{
  Sameep Chattopadhyay$^{\dagger}$ \quad
  Nikhil Karamchandani$^{\diamond}$ \quad
  Sharayu Moharir$^{\diamond}$ \\[2mm]
  $^{\dagger}$Paul G.\ Allen School of Computer Science \& Engineering, University of Washington \\
  $^{\diamond}$Department of Electrical Engineering, Indian Institute of Technology Bombay \\
  \texttt{sameepch@uw.edu} \quad
  \texttt{nikhilk@ee.iitb.ac.in} \quad
  \texttt{sharayum@ee.iitb.ac.in}
}

\begin{document}

\maketitle
\vspace*{-1.5em}
\begin{abstract}
\vspace*{-0.5em}
Online learning to rank (OLTR) plays a critical role in information retrieval and machine learning systems, with a wide range of applications in search engines and content recommenders. However, despite their extensive adoption, the susceptibility of OLTR algorithms to coordinated adversarial attacks remains poorly understood. In this work, we present a novel framework for attacking some of the widely used OLTR algorithms. Our framework is designed to promote a set of target items so that they appear in the list of top-$\nreco$ recommendations for $T - o(T)$ rounds, while simultaneously inducing linear regret in the learning algorithm. We propose two novel attack strategies: $\ofcs$ for $\cascadeucb$ and $\ofpb$ for $\pbmucb$. We provide theoretical guarantees showing that both strategies require only $O(\log T)$ manipulations to succeed. Additionally, we supplement our theoretical analysis with empirical results on real-world data.
\end{abstract}
\vspace*{-0.5em}
\section{Introduction}
\vspace*{-0.5em}
\textit{Online Learning to Rank} (OLTR) \cite{10.1145/2911451.2914798} is a sequential decision-making framework widely used in information retrieval systems \cite{liu2009learning} to rank items based on user feedback. In an OLTR setup, the learning agent presents a ranked list of items to the user in each round, and the user provides implicit feedback by interacting with the list in some manner. One of the most widely studied forms of such feedback is click behavior, which the agent seeks to model and optimize in order to improve its recommendations. Over the years, several click feedback models have been studied, with two of the most prominent being the \textit{Cascade} model \cite{cascademod}, where the user examines the list sequentially from top to bottom and clicks the first item they find ``attractive'', therefore resulting in a single click per round, and the \textit{Position-based model} \cite{chuklin2015click}, which allows multiple clicks in a round and assumes that the probability of an item being examined depends on its position in the list.

Due to the sequential nature of OLTR, several studies have attempted to reformulate the problem within the framework of \textit{Multi-Armed Bandits} (MABs) \cite{lattimore2020bandit}. Recent works have developed efficient learning algorithms for the Cascade  \cite{kveton15cascade, 10.5555/2969239.2969401, Zong2016CascadingBF} and the Position-based \cite{PBM-UCB} feedback models. Additionally, more general algorithms that accommodate a broader class of user feedback structures, encompassing both the above models, have also been proposed and analyzed \cite{10.5555/3305890.3306115, Lattimore2018TopRankAP}.  

Given the widespread use of online decision-making systems, understanding their response to coordinated attacks is crucial. While extensive research has explored adversarial attack strategies for the MABs, with frameworks like reward manipulation and action manipulation  \cite{Jun2018AdversarialAO,NEURIPS2021_be315e7f,Liu2019DataPA,Wang2024StealthyAA,Rangi2021SavingSB}, relatively little attention has been given to such attacks in the OLTR setting \cite{10.5555/3666122.3667927,wang2024adversarial}.

This work focuses on a specific class of reward manipulation strategies known as observation-free attacks \cite{NEURIPS2021_be315e7f}, where a collective adversary generates (or manipulates) reward signals to promote a set of pre-determined items in a learning algorithm without having access to the algorithm’s feedback. A key feature of these observation-free attacks is that they can be executed in the real world by groups of ordinary individuals who, in spite of lacking access to the feedback from other users, can influence the algorithm to promote unattractive yet socially beneficial content through their own coordinated actions. This makes them a compelling instance of \textit{algorithmic collective action} \cite{10.5555/3618408.3618918,10.5555/3737916.3741700}.

Despite their real-world applications and proven success for MABs \cite{NEURIPS2021_be315e7f}, no effective observation-free attack strategy has yet been demonstrated for even the most common OLTR algorithms \cite{wang2024adversarial}.
\label{sec:introduction}
\vspace*{-0.5em}
\subsection{Our Contributions}
\vspace*{-0.5em}
In this paper, we present the first study on observation-free attacks against OLTR algorithms with click feedback.  Our technical contributions are summarized as follows.
 In Section \ref{sec:genattack}, we introduce a general framework for a limited-time observation-free attack strategy to promote a set of target items in OLTR.
 Based on this framework, we propose two novel attack strategies, namely $\ofcs$ and $\ofpb$ for the Cascade and the  Position-based feedback models, respectively.
 We prove that both attack strategies require just \(O(\log T)\) reward manipulations to successfully promote their target elements for \(T - o(T)\) rounds and impose \(\Omega(T)\) regret on their respective learning algorithms. 
 Our work also addresses some of the major drawbacks of the earlier attack strategies for OLTR, specifically the need for continuous manipulation of rewards in \cite{10.5555/3666122.3667927} and the inability to attack UCB-based OLTR algorithms through reward manipulation in \cite{wang2024adversarial}.
 We supplement our analysis with empirical evaluation of our strategies using the MovieLens dataset \cite{Harper2016TheMD}.
    
\vspace*{-0.5em}
\subsection{Related Work}
\vspace*{-0.5em}
Numerous recent studies have explored adversarial attacks on classical multi-armed bandit algorithms \cite{Jun2018AdversarialAO,NEURIPS2021_be315e7f,Liu2019DataPA,Wang2024StealthyAA,Rangi2021SavingSB}. However, adversarial attacks on OLTR \cite{10.5555/3666122.3667927,wang2024adversarial} remain relatively unexplored.

While most of the attacks on MABs could potentially be extended to OLTR, the combinatorial action space and restricted feedback structure in OLTR make these extensions non-trivial \cite{10.5555/3666122.3667927}. To the best of our knowledge, only two studies have investigated adversarial attacks on OLTR algorithms. The first one \cite{10.5555/3666122.3667927}, introduced a reward manipulation attack on $\cascadeucb$ and $\pbmucb$, demonstrating that a target item could be recommended for \(T - o(T)\) rounds with only \(o(T)\) corruptions. A key limitation of this attack was its reliance on observing rewards and performing costly computations in every round. The second study \cite{wang2024adversarial} proposed an attack-then-quit (ATQ) strategy for OLTR algorithms based on item elimination. The study also discusses the hindrances in applying ATQ-like attacks to UCB-based OLTR algorithms as one of its major drawbacks.
\vspace*{-0.5em}
\section{Setting}
\label{sec:stoch}
\vspace*{-0.5em}
In this section, we discuss our problem setting for designing an observation-free attack on OLTR algorithms. For the ease of notation, we define the following quantities.  Let the universal set of all available items be \(\totset = \{1, 2, \dots, \ntot\}\); at each time-step (round) \(t\), the user is presented with an ordered list \(\reclist_t = (a_{1,t}, \dots, a_{\nreco,t})\) of \(\nreco\) items selected from \(\totset\),  where \(\nreco \leq \ntot \). Here, \( a_{i,t} \) is the item present at the \( i^{\mathrm{th}} \) position of \( \reclist_t \), i.e., $ \reclist_t[i]= a_{i,t}$. The user examines \( \reclist_t \) and provides feedback to the OLTR algorithm through clicks, following a specified click feedback model. 

\vspace*{-0.5em}
\subsection{Click Feedback Models}
\vspace*{-0.5em}
This work primarily focuses on two of the most common click feedback models for OLTR: the Cascade model \cite{cascademod} and the Position-based model  \cite{chuklin2015click}. In both models, the items are characterized by attraction probabilities \(\wts \in [0, 1]^{\ntot}\). When a user examines an item \(a=a_{i,t}\) located at position \(i\) in list \(\reclist_t\) during round \(t\), the probability of clicking on it is given by \(\wts_{a}\), which is independent of previous rounds and the attraction probability of other items. To formally characterize the user feedback, we define the following quantities: 
\begin{enumerate}[left=2pt]
 \item[--] \textit{Examination Feedback}: $\exam_t \in \{0,1\}^{\nreco}$, with $\exam_{i,t}=1$, if and only if  $a_{i,t}$ is examined in round $t$.
\item[--] \textit{Click Feedback}: $\click_t \in \{0,1\}^{\nreco}$, where $\click_{i,t}=1$, if and only if $a_{i,t}$ is clicked in round $t$.
\end{enumerate}
Using the above-defined quantities, we now describe the following click feedback models in detail.
\vspace*{-0.5em}
\subsubsection{Cascade Model}
\vspace*{-0.5em}
\label{subsec:casmodel}
The Cascade model \cite{cascademod} is one of the earliest and most well-studied user feedback models. In this model, the user examines  \(\reclist_t\) sequentially from top to bottom, selecting the first item they find \textit{attractive}, which, in the context of web search, is indicated by a click. Once an item \(a_{i,t}\) is clicked, the user concludes the search, leaving all the subsequent items unexamined for that round. Conversely, if \(a_{i,t}\) fails to attract the user, they proceed to examine the next item \(a_{i+1,t}\), and continue to do so until the first attractive item is found or $\reclist_t$ is completely examined.   The Cascade model assumes that only a single item can be clicked in any given round.
\vspace*{-0.5em}
\subsubsection{Position-based Model (PBM)}
\label{subsec:pbm}
\vspace*{-0.5em}
The \textit{Position-based model} \cite{chuklin2015click} serves as a prominent alternative to the Cascade model for simulating user behavior in OLTR. Unlike the Cascade model, where the items are sequentially examined, the PBM suggests that the likelihood of examination is influenced by the \textit{position bias}, captured by $\poslist = \left(p_1, \dots, p_\nreco\right)$. Specifically, an item appearing in the $i^\mathrm{th}$ position of $\reclist_t$ is examined with a probability $p_i$, independently of the other positions and their corresponding items. The examination probabilities are assumed to be constant through time and decreasing as the position moves down the list, with $p_1 \geq p_2 \geq \dots \geq p_\nreco$.

In PBM, for item \( a=a_{i,t} \), the click feedback is given by
    $\click_{i,t} = \exam_{i,t} \cdot Y_{a,t}$, where $\exam_{i,t}\sim \mathrm{Bernoulli}(p_i)$ represents the examination feedback of  \(a_{i,t} \) in round t, while $Y_{a,t}\sim \mathrm{Bernoulli}(\wts_{a})$ reflects the clicking of item $a_{i,t}$, conditioned on it being examined. A key difference between the PBM and the Cascade model is that PBM allows for multiple items in $\reclist_t$ to be clicked within a single round. 

\color{black}
\vspace*{-0.5em}
\subsection{OLTR with Click Feedback}
\label{sec:casbandit}
\vspace*{-0.5em}
For real-world use cases, the attractiveness of the items is often unknown, and the recommenders have to learn them from the user feedback, which brings us to the problem of online learning to rank. As per the OLTR setup, at each time step \( t \), the learning agent recommends a list \( \reclist_t \) to the user and observes the corresponding feedback. If a user clicks on the item \( a = a_{i,t} \),  i.e.,  \( \click_{i,t} = 1 \),  the agent receives a click feedback for the position, and a reward \( \rew_{a,t} = 1 \) for the item. If an item is not clicked, the agent receives \( \rew_{a,t} = 0 \). Therefore, in the absence of any external manipulations, {\[\displaystyle
    \rew_{a,t} = \sum^\nreco_{i=1}  \click_{i,t} \cdot \mathds{1}{\left\{a=a_{i,t}\right\}}.
\]}
Given the attraction probabilities 
$w$, the expected reward in round  $t$ is denoted by $f(\reclist_t,w)$, which equals the expected number of clicks received by the given list:{
\[
f(\reclist_t,w)=\mathds{E}\left[\sum_{a \in \reclist_t }\rew_{a,t}\right]= \mathds{E}\left[\sum_{a \in \reclist_t }\sum^\nreco_{i=1}  \click_{i,t} \cdot \mathds{1}{\left\{a=a_{i,t}\right\}}\right]= \mathds{E}\left[ \lVert\click_{t}\rVert_1\right].
\]}


At each round, the learning agent aims to maximize its expected reward by recommending the most attractive items as a part of $\reclist_t$. Without loss of generality, let the items in \( \totset \) be indexed in decreasing order of their attraction probabilities. We define \( \reclist^\ast = (1, \dots, \nreco) \) as the list of the \( \nreco \) most attractive items, ordered by descending attractiveness. An optimal static policy, recommending $\reclist^\ast$ to the user at all rounds, maximizes the expected reward under both the click models. Any learning policy is evaluated on the \textit{cumulative expected regret} up to horizon $T$, which is given by {
\[
  \regret(T) = \mathds{E}\left[ \sum_{t = 1}^T \reglist(\reclist_t, \wts)\right], \text{ where $ \reglist(\reclist_t, \wts)= f(\reclist^\ast,\wts)- f(\reclist_t,\wts)$.}
\]}
 Over the past few years, various algorithms \cite{kveton15cascade,10.5555/2969239.2969401,PBM-UCB,10.5555/3305890.3306115,Lattimore2018TopRankAP} have been proposed for OLTR with click feedback. Details for two such UCB-based OLTR algorithms, namely $\cascadeucb$ (Algorithm \ref{alg:cascadeucb1}) for the Cascade model and  $\pbmucb$ (Algorithm \ref{alg:pbmucb}) for PBM, are provided in Appendix \ref{App:Algos}.
\vspace*{-0.5em}
\subsection{Attacking OLTR Algorithms}
\vspace*{-0.5em}
Since OLTR algorithms optimize exclusively for click feedback, they often overlook items that are less likely to attract clicks, regardless of their other attributes. Our study examines strategies that a group of adversarial users could employ to promote such less attractive items in OLTR. We focus on a class of adversarial strategies known as \textit{observation-free attacks} \cite{NEURIPS2021_be315e7f}, in which users manipulate their reward vector $\rew_t = \{\rew_{a,t}\}_{a \in \totset}$ into $\hat{\rew}_t$ without access to any feedback or rewards from other users (i.e., $\exam_t,\, \click_t,\, \rew_t$ for $t \in \{1,2\dots,T\}$), in order to promote their target items.

\vspace*{-0.5em}\section{Observation-Free Attack Strategies for OLTR}
\label{sec:genattack}
\vspace*{-0.5em}
Our attack strategies for OLTR are inspired by the attack given in \cite{NEURIPS2021_be315e7f} for the classical MAB algorithms. Given a set $\tarset$  of target items containing $N \leq \nreco$ items, we propose a three-phase attack framework. The attack is initialized with a list $\tarlist$ that includes all items in $\tarset$, together with $\nreco - N$ items arbitrarily selected from $\totset \setminus \tarset$, and it begins at the very start of the learning process. In the first phase, lasting $\ca$ rounds, the collective adversary transmits zero rewards for all items. In the second phase, for the next $\cb$ rounds, it assigns positive rewards to a subset of items in $\tarlist$ while forcing zero rewards for the rest. In the third phase, the adversary applies no further manipulations.

\begin{wrapfigure}{r}{0.5\linewidth}
\vspace{-2.5em}
    \begin{minipage}{\linewidth}
        \begin{algorithm}[H]
        \caption{$\ofcs$}
        \label{alg:observation_free_cascade}
        \begin{algorithmic}
        \State {\bfseries Input:} Horizon $T$, Item set $\totset$, and Target set $\tarset$.
        \State \textbf{Initialize} list $\tarlist$ containing all elements in $\tarset$. 
        \State \textbf{Calculate} $\pmy$, $\ca$ and $\cb$  as per Section \ref{subsec:cas-skel}. 
        \For{$t=1 \dots T$}
            \If{$t \leq \ca$}
                \State $\hat{\rew}_{a,t} = 0 \; \forall a \in \totset$
            \ElsIf{$\ca < t \leq \ca+\cb$}
                \State $i = \left\lceil \frac{\nreco(t-\ca)}{\cb} \right\rceil$
                \State Set $\hat{\rew}_{a,t} = \mathds{1}{\{a = \tarlist[i],a\in \reclist_t\}}, \; \forall\, a$
            \Else
                \State $\hat{\rew}_t = \rew_t$
            \EndIf
        \EndFor
        \end{algorithmic}
        \end{algorithm}
    \end{minipage}
    \vspace{-2em}
\end{wrapfigure}
 The first phase of the attack forces all items in $\totset$ to have low empirical reward estimates, while the second phase ensures that the items in $\tarlist$ are brought into $\reclist_t$ and are significantly differentiated from the rest.
 The attack actively occurs only in the first $\ca+\cb$ rounds, while ensuring with  high probability that the OLTR algorithm recommends a permutation of $\tarlist$ to the users in all the subsequent rounds until horizon $T$. Assuming $\ca+\cb \ll T$, such an attack can be conceptualized as the rapid injection of misleading feedback by a group of adversarial users at the start of the learning process.

   We implement the aforementioned framework for designing attacks on two UCB-based OLTR algorithms: $\cascadeucb$ \cite{kveton15cascade} and $\pbmucb$ \cite{PBM-UCB}. UCB-based OLTR algorithms are known to be quite resistant to such limited-time attack strategies \cite{wang2024adversarial}. This resistance arises because, once the adversary stops, the true feedback for other items is gradually revealed to the algorithm, allowing the learner to recognize that the targeted items are not the most attractive ones. To prevent this, by the end of the second phase, our framework ensures that all the non-target items have sufficiently low UCBs, and thus, with a high probability,  they do not receive any further examinations (and clicks) even after the manipulations conclude.


\vspace*{-0.5 em}
\subsection{Attacking CascadeUCB1}
\vspace*{-0.5 em}
\label{subsec:casattackucb}
To efficiently attack the $\cascadeucb$ algorithm and promote items from a given target set $\tarset$, we propose the 
$\ofcs$ strategy outlined in Algorithm \ref{alg:observation_free_cascade}, with parameters specified in Section \ref{subsec:cas-skel}. The analytical results for this strategy are presented in Theorem \ref{thm:cascadeatk}.
\vspace*{-0.5 em}
\subsubsection{CascadeOFA: Skeleton and Details}
\vspace*{-0.5 em}
$\ofcs$ requires initializing a list $\tarlist$ of length $\nreco$ that contains all the items of $\tarset$. Following this, we define an attack parameter $\pmy$, which is used to specify the phase durations, with
\[ 
\pmy  = (1-\epsilon) \min\left\{\frac{1}{\nreco},w_{\mathrm{min}}\right\}, \text{where  $\epsilon \in (0,1)$ and $\displaystyle w_{\mathrm{min}}= \min_{a \in \tarlist} \wts_a$. }
\] 
Note that  \( \ofcs \) does not require knowledge of the exact values of $\wts_a$'s; any $\pmy$ smaller than \(\min\left\{1/\nreco, w_{\mathrm{min}}\right\} \) is sufficient to launch the attack. An observation-free attack on $\cascadeucb$ with a parameter $\alpha>1$ (described in Remark \ref{rem:alpha}) proceeds in the following manner:

\label{subsec:cas-skel}
 \textbf{Phase 1.} Set $\hat{\rew}_{a,t}=0 \;\forall\; a \in \totset$ for the first $\ca$ rounds, where 
$\ca = \ntot\left\lceil \frac{ \alpha \log T } {\nreco\psq}  \right\rceil \;.$

\noindent\textbf{Phase 2.} The second phase lasts for $\cb$ rounds with \(
\cb = \nreco\left\lceil \frac{\pmy \nreco\ca/\ntot + \ntot-\nreco+1}{1-\nreco\pmy} \right\rceil\;.\)

Split this phase into $\nreco$  sub-phases of equal length with the $i^{\text{th}}$ sub-phase lasting from round $\ca+ (i-1)\cb/{\nreco}+1$ to $\ca+{i\cb}/{\nreco}$. In the $i^\text{th}$ sub-phase, enforce $\hat{\rew}_{a,t}= \mathds{1}{\{a = \tarlist[i],a\in \reclist_t\}}$.


\noindent\textbf{Phase 3.}  No further reward manipulation is applied in the remaining rounds of $\cascadeucb$.
\subsubsection{Theoretical Analysis}
\label{sec:theo-cas}
\begin{theorem}
\label{thm:cascadeatk} 
If a collective adversary attacks $\cascadeucb$ using the $\ofcs$ strategy outlined in Algorithm \ref{alg:observation_free_cascade}, then with $O(\log T)$ reward manipulation, it can ensure that each item $\arm \in \tarset$ is recommended for at least $T - O(\log T)$ rounds, with probability exceeding $1 - \frac{\nreco}{T}$, thereby imposing $\regret(T) = \Omega(T)$ (linear regret) on $\cascadeucb$.
\end{theorem}
A detailed proof of Theorem \ref{thm:cascadeatk} is provided in Appendix \ref{app:proof-cas}.

\vspace*{-0.5 em}
\subsection{Attacking PBM-UCB}
\vspace*{-0.5 em}
\label{subsec:pbmatkucb}
\begin{wrapfigure}{r}{0.5\linewidth}
\vspace{-2.5em}
    \begin{minipage}{\linewidth}
        \begin{algorithm}[H]
        \caption{$\ofpb$}
        \label{alg:observation_free_pbm}
        \begin{algorithmic}
        \State {\bfseries Input:} Position bias $\poslist$, Horizon $T$, Item set $\totset$, Target set $\tarset$.
        \State \textbf{Initialize} list $\tarlist$ containing all elements in $\tarset$. 
        \State \textbf{Calculate} $\ca$ and $\cb$ as per Section \ref{subsec:pbm-skel}.
        \For{$t=1 \dots T$}
            \State $\hat{\exam}_{i,t} \sim \mathrm{Bernoulli}(p_i)\; \forall i = 1,\dots, \nreco$
            \If{$t \leq \ca$}
                \State $\hat{\rew}_{a,t} = 0 \; \forall a \in \totset$
            \ElsIf{$\ca < t \leq \ca+\cb$}
                \State $\displaystyle \hat{\rew}_{a,t} =\sum^\nreco_{i=1} \hat{\exam}_{i,t} \cdot \mathds{1}{\{a = a_{i,t},a \in \tarlist\}}$
            \Else
                \State $\hat{\rew}_t = \rew_t$
            \EndIf
        \EndFor
        \end{algorithmic}
        \end{algorithm}
    \end{minipage}
    \vspace{-4.5em}
\end{wrapfigure}
Following the approach of $\ofcs$, we propose the $\ofpb$ strategy outlined in Algorithm \ref{alg:observation_free_pbm} for attacking $\pbmucb$ with exploration parameter $\alpha>1$ (Remark \ref{rem:alpha}) and position bias \(\poslist = (p_1, \dots, p_\nreco)\);  we define $\lambda_p =p_1/p_\nreco$.

At each round of the attack, the adversary generates $\hat{\exam}_t\in \{0,1\}^\nreco$ with $\hat{\exam}_{i,t} \sim \mathrm{Bernoulli}(p_i)$ to examine $\reclist_t$ according to PBM. 

Similar to $\ofcs$, this attack also requires a parameter $\pmy\in$ \((0,\min\left\{1/\lambda_p, \wts_{\mathrm{min}}\right\}) \). Owing to the unique characteristics of position-based models, the computation of $\ca$ and $\cb$ differs markedly from that for the Cascade model, as detailed in Section \ref{subsec:pbm-skel}. The consequences of the $\ofpb$ strategy on $\pbmucb$ are presented in Theorem \ref{thm:pbmatk}.

\vspace*{-0.5 em}
\subsubsection{PBMOFA: Skeleton and Details}
\vspace*{-0.5 em}
\label{subsec:pbm-skel}
 \noindent\textbf{Phase 1.} Set $\hat{\rew}_{a,t}=0 \;\forall\; a \in \totset$ for the first $\ca$ rounds where \(
      \ca= \left\lceil\frac{L }{\nreco} \left\{\frac{\lambda^2_p\alpha\log T}{\psq p^2_\nreco} + 1\right\}   \right\rceil.\)

\noindent \textbf{Phase 2.} Enforce $\hat{\rew}_{a,t}= \sum^\nreco_{i=1} \hat{\exam}_{i,t} \cdot  \mathds{1}{\left\{a=a_{i,t},a \in \tarlist\right\}}$, for the next $\cb$ rounds where
{ \small\[
  \cb=\left\lceil\frac{L \left(\lambda^2_p\gamma\log T + 1\right)}{\nreco}  \right\rceil, \text{with }  
\gamma=\frac{2\rho\eta+1 + \sqrt{4\rho\eta+1} }{2\eta^2},\,
\rho= p_1\pmy\left\{\frac{\nreco\ca}{\log T}-\frac{(L-1)\alpha}{\psq p^2_\nreco}\right\},
\]} and 
$\eta= p_\nreco-p_1\pmy.$ 

\noindent \textbf{Phase 3.} The third phase of $\ofpb$ proceeds  without any further reward manipulation.

In contrast to the earlier reward-manipulation attacks on $\pbmucb$ \cite{10.5555/3666122.3667927,wang2024adversarial}, $\ofpb$ explicitly incorporates the probabilistic examination structure of the Position-based model while generating $\hat{\exam}_t$.
\vspace*{-1.5em}
\subsubsection{Theoretical Analysis}
\label{sec:theo-pbm}
\begin{theorem}
\label{thm:pbmatk} 
If a collective adversary attacks $\pbmucb$ using the $\ofpb$ strategy outlined in Algorithm \ref{alg:observation_free_pbm}, then with $O(\log T)$ reward manipulations, it can ensure that each item $\arm \in \tarset$ is recommended to the users for at least $T - O(\log T)$ rounds, with probability exceeding $1 - \frac{2\nreco}{T}$, thereby imposing $\regret(T) = \Omega(T)$ (linear regret) on $\pbmucb$.
\end{theorem}
A detailed proof of Theorem \ref{thm:pbmatk} is provided in Appendix \ref{app:proof-pbm}.
\vspace*{-0.5em}
\section{Empirical Results}
\vspace*{-0.5em}
\label{sec:emp}
To support our theoretical analysis, we conduct experiments on the MovieLens dataset \cite{Harper2016TheMD}, containing ratings for about 3,900 movies. Following \cite{Zong2016CascadingBF}, we assign a reward of one when a user rates a movie above three stars. For a given $\ntot$, $\totset$ is created by arbitrarily choosing $\ntot$ movies from the dataset.

In all experiments, we set $T = 5\times10^5$, $\nreco = 3$, $\ntot = 10$, and average results over 50 runs. For PBM, we use $\poslist = (0.95, 0.90, 0.85)$. With $\tarlist = (4,7,10)$ for $\ofcs$ and $\tarlist = (8,9,10)$ for $\ofpb$, we obtain $\pmy \approx 0.08$ for both strategies. Under these parameters, the required reward manipulation $(\ca+\cb)$ is $11265$ for $\ofcs$ and $12811$ for $\ofpb$, both well under $3\%$ of the total number of rounds. To evaluate the effectiveness of our attacks, we compare the number of recommendations for the items in $\reclist^\ast=(1,2,3)$ and $\tarlist$, with and without manipulation.

Figure \ref{fig:success} presents the results for $\ofcs$ and $\ofpb$; in the absence of any manipulation, both $\cascadeucb$ and $\pbmucb$ recommend the items in $\reclist^\ast $ for most rounds. In contrast, with our attacks, the target items receive the maximum recommendations for both the OLTR algorithms. 

Next, we compare the performance of $\ofcs$ with the following baselines:
\vspace*{-0.5em}
\begin{enumerate}[left=2pt]
    \item \textbf{No Attack}: The $\cascadeucb$ algorithm without any external manipulation.
    \item \textbf{CascadeATQ}: A naive attack-then-quit strategy, introduced by  \cite{wang2024adversarial}, with the same amount of reward manipulation as $\ofcs$. For $t \in \{1,\dots\ca+\cb\}$, the adversary clicks $\left(\hat{\rew}_{a,t}=1\right)$ the top-most target item in $\reclist_t$; if no such items exist, it ignores $\left(\hat{\rew}_{a,t}=0\right)$ all the items in $\reclist_t$.
    \item \textbf{CascadeAlphaAtk} (Algorithm 4 in \cite{10.5555/3666122.3667927} ): First  reward-manipulation attack on $\cascadeucb$. 
\end{enumerate}
\vspace*{-0.5em}
\begin{figure}[h]
\centering
\begin{subfigure}[h]{0.49\linewidth}
    \centering
    \includegraphics[width=\linewidth]{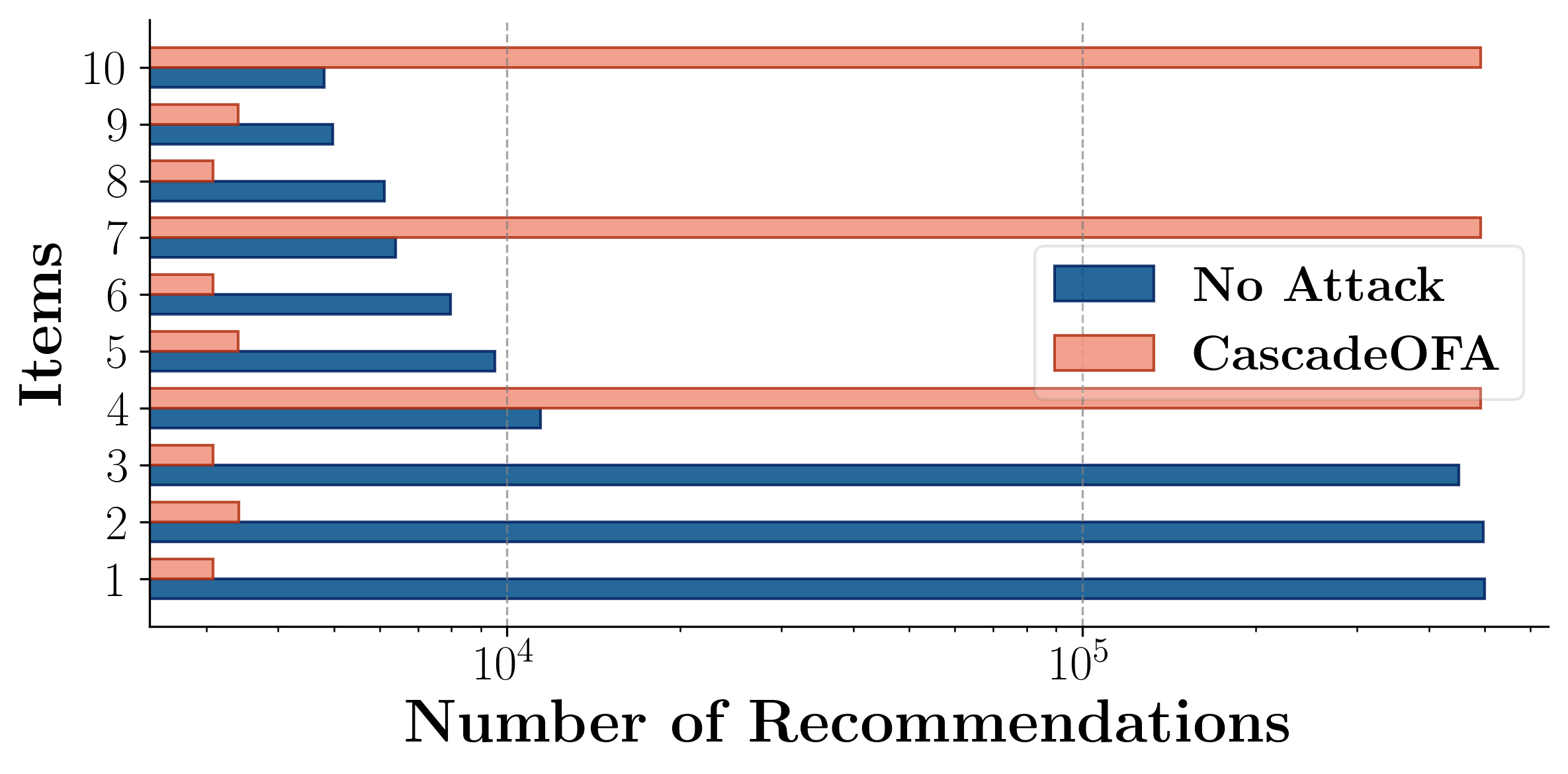}
    \subcaption{$\ofcs$ with $\tarlist=(4,7,10)$.}
    \label{fig:success_cas}
\end{subfigure}
\hfill
\begin{subfigure}[h]{0.49\linewidth}
    \centering
    \includegraphics[width=\linewidth]{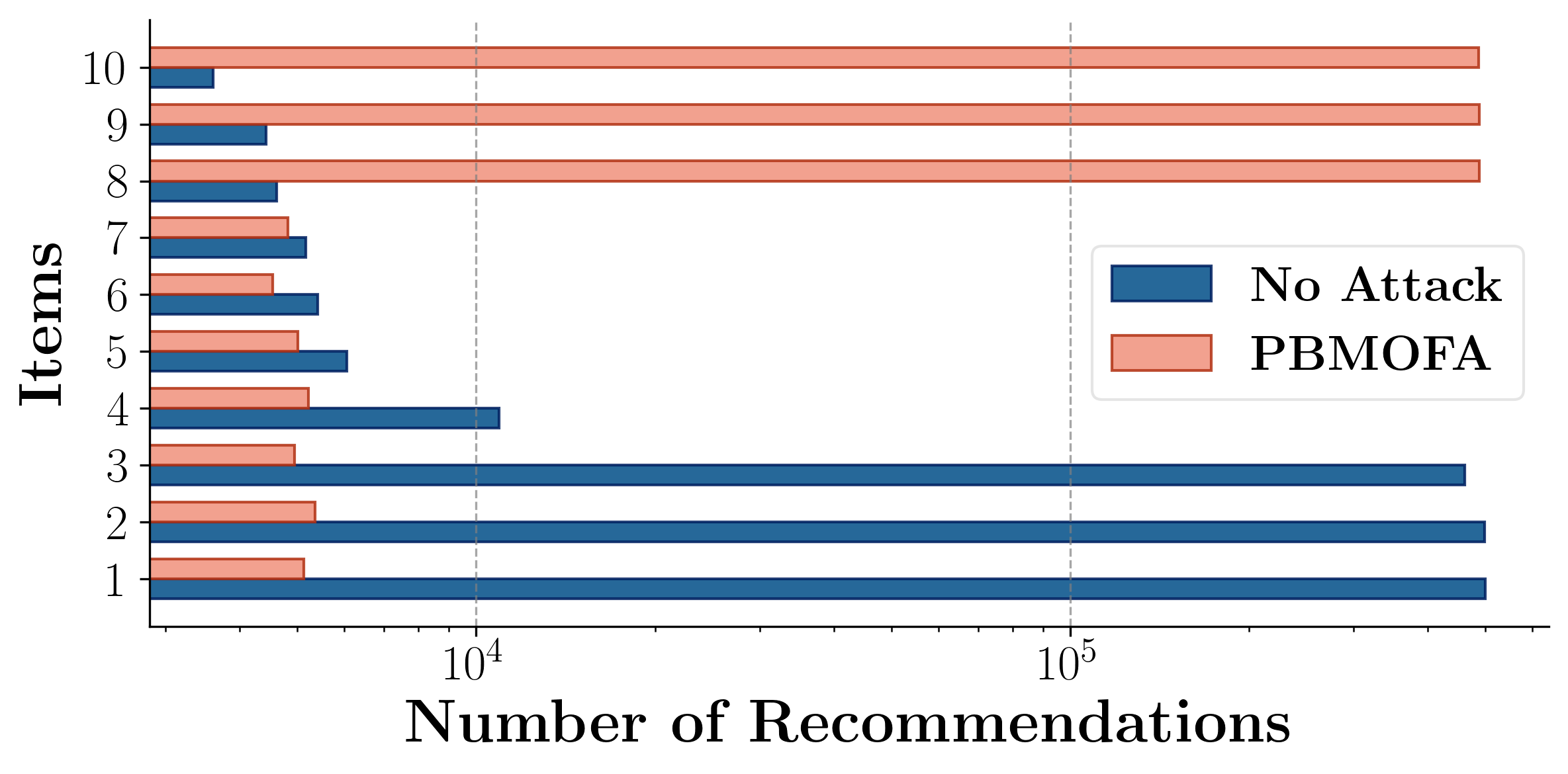}
    \subcaption{$\ofpb$ with $\tarlist=(8,9,10)$ }
    \label{fig:success_pbm}
\end{subfigure}
\caption{Comparison of the number of recommendations for target items with and without attack.}
\vspace*{-0.5em}
\label{fig:success}
\end{figure}

We similarly compare the performance of $\ofpb$ with  PBMAlphaAtk (Algorithm 2 in \cite{10.5555/3666122.3667927}) and a naive PBMATQ strategy. The naive strategy is synonymous with that of the collective adversary clicking all the examined target items while ignoring the rest, i.e.,  $\hat{\rew}_{a,t}= \sum^\nreco_{i=1} \hat{\exam}_{i,t} \cdot  \mathds{1}{\left\{a=a_{i,t},a \in \tarlist\right\}}$, in each round until $t=\ca+\cb$. Additional details on the reward manipulations required for all the attack strategies discussed above are provided in Appendix \ref{app:emp}.

Figure \ref{fig:regret} compares the regret enforced by the aforementioned attacks over $\cascadeucb$, and $\pbmucb$. Both CascadeAlphaAtk and $\ofcs$ impose linear regret on $\cascadeucb$, whereas CascadeATQ fails to do so despite having the same amount of reward manipulations as $\ofcs$. Similar trends follow for the Position-based model as well. 

It is important to note that comparing CascadeAlphaAtk and PBMAlphaAtk with other strategies is inherently unfair, as they have access to significantly more information in the form of user feedback and rewards. Both attack strategies operate online, with actions determined by user feedback.

\begin{figure}[h]
\vspace{-0.5em}
\begin{subfigure}[h]{0.49\linewidth}
    \centering
    \includegraphics[width=\linewidth]{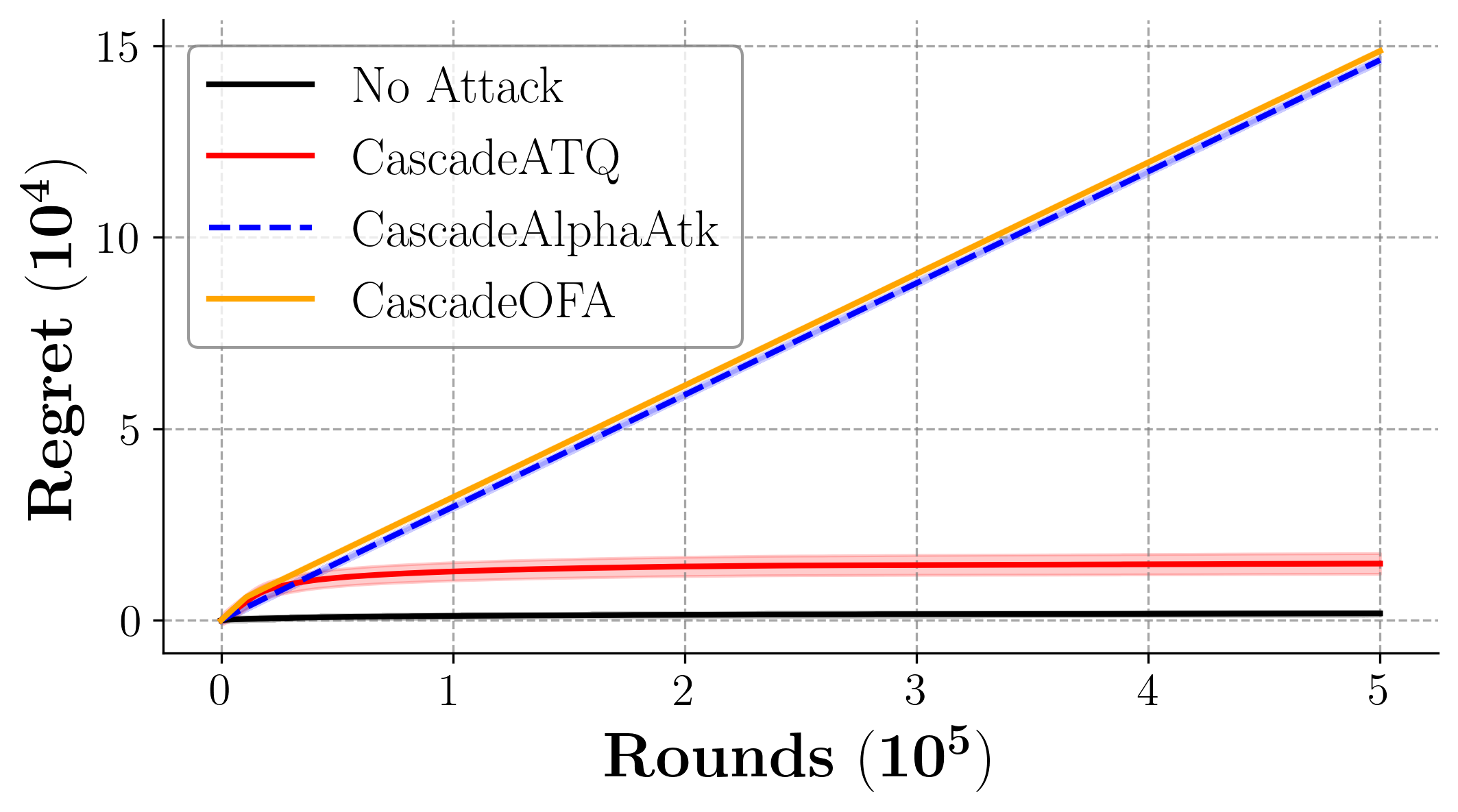}
    \subcaption{Under $\cascadeucb$}
    \label{fig:regret_cas}
\end{subfigure}
\hfill
\begin{subfigure}[h]{0.49\linewidth}
    \centering
    \includegraphics[width=\linewidth]{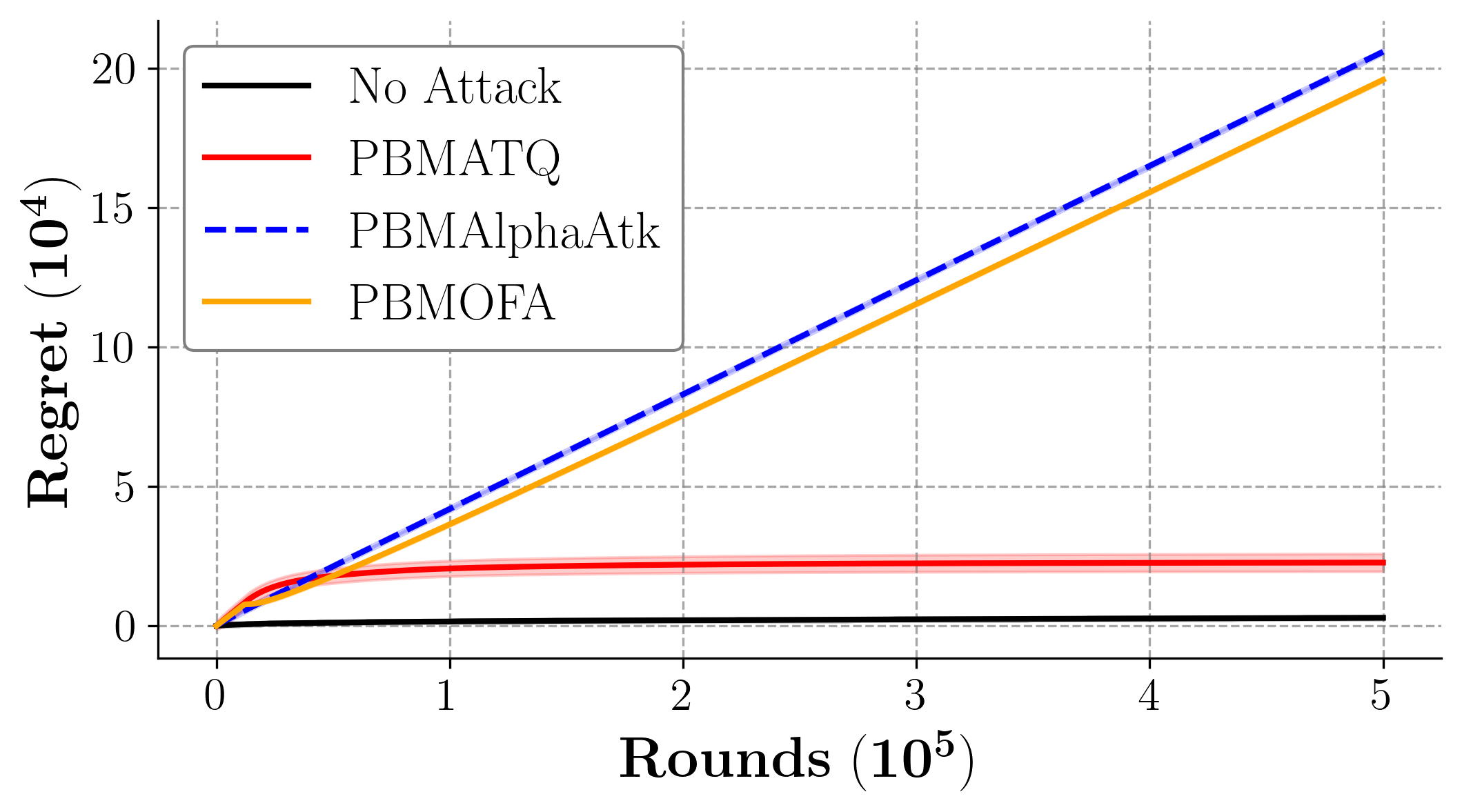}
    \subcaption{Under $\pbmucb$}
    \label{fig:regret_pbm}
\end{subfigure}
\centering
\caption{Comparison of regret for attack strategies under different click feedback models.}
\label{fig:regret} 
\end{figure}

\begin{remark}
We have conducted empirical experiments with some of the other popular OLTR algorithms as well, with results shown in Figure \ref{fig:success1} in Appendix \ref{app:emp}. We observe that $\ofcs$ successfully enforces the target choices for CascadeKL-UCB \cite{kveton15cascade} and TS-Cascade \cite{10.5555/3546258.3546476} as well. \end{remark}
\vspace*{-0.5em}
\vspace*{-0.8em}
\section{Conclusions}
\vspace*{-0.7em}
In this paper, we presented the first observation-free attack on OLTR algorithms across different click feedback models. We design two specific attacks: $\ofcs$ for $\cascadeucb$ and $\ofpb$ for $\pbmucb$. Our analysis demonstrates that both attacks, can successfully promote their target items for $ T-o(T)$ rounds with high probability, while requiring only $O(\log T)$ reward manipulations. We have supported our analysis through experiments on the real-world MovieLens dataset. In future work, we aim to develop attacks on OLTR algorithms under general stochastic click models and explore novel OLTR algorithms that would be inherently robust against such manipulations.

\bibliography{ref}

@inproceedings{10.5555/3618408.3618918,
author = {Hardt, Moritz and Mazumdar, Eric and Mendler-D\"{u}nner, Celestine and Zrnic, Tijana},
title = {Algorithmic collective action in machine learning},
year = {2023},
publisher = {JMLR.org},
booktitle = {Proceedings of the 40th International Conference on Machine Learning},
articleno = {510},
numpages = {17},
location = {Honolulu, Hawaii, USA},
series = {ICML'23}
}

@inproceedings{10.5555/3737916.3741700,
author = {Baumann, Joachim and Mendler-D\"{u}nner, Celestine},
title = {Algorithmic collective action in recommender systems: promoting songs by reordering playlists},
year = {2025},
isbn = {9798331314385},
publisher = {Curran Associates Inc.},
address = {Red Hook, NY, USA},
booktitle = {Proceedings of the 38th International Conference on Neural Information Processing Systems},
articleno = {3784},
numpages = {27},
location = {Vancouver, BC, Canada},
series = {NIPS '24}
}

@inproceedings{kveton15cascade,
	Author = {Kveton, Branislav and Wen, Zheng and Ashkan, Azin and Szepesvari, Csaba},
	Booktitle = {ICML},
	Pages = {767--776},
	Title = {Cascading Bandits: Learning to Rank in the Cascade Model},
	Year = {2015}}

@inproceedings{NEURIPS2021_be315e7f,
 author = {Xu, Yinglun and Kumar, Bhuvesh and Abernethy, Jacob D},
 booktitle = {Advances in Neural Information Processing Systems},
 editor = {M. Ranzato and A. Beygelzimer and Y. Dauphin and P.S. Liang and J. Wortman Vaughan},
 pages = {22550--22561},
 publisher = {Curran Associates, Inc.},
 title = {Observation-Free Attacks on Stochastic Bandits},
 url = {https://proceedings.neurips.cc/paper_files/paper/2021/file/be315e7f05e9f13629031915fe87ad44-Paper.pdf},
 volume = {34},
 year = {2021}
}

@inproceedings{cascademod,
author = {Craswell, Nick and Zoeter, Onno and Taylor, Michael and Ramsey, Bill},
title = {An experimental comparison of click position-bias models},
year = {2008},
isbn = {9781595939272},
publisher = {Association for Computing Machinery},
address = {New York, NY, USA},
url = {https://doi.org/10.1145/1341531.1341545},
doi = {10.1145/1341531.1341545},
booktitle = {Proceedings of the 2008 International Conference on Web Search and Data Mining},
pages = {87–94},
numpages = {8},
series = {WSDM '08}
}

@article{10.1023/A:1013689704352,
author = {Auer, Peter and Cesa-Bianchi, Nicol\`{o} and Fischer, Paul},
title = {Finite-time Analysis of the Multiarmed Bandit Problem},
year = {2002},
issue_date = {May-June 2002},
publisher = {Kluwer Academic Publishers},
address = {USA},
volume = {47},
number = {2–3},
issn = {0885-6125},
url = {https://doi.org/10.1023/A:1013689704352},
doi = {10.1023/A:1013689704352},
journal = {Mach. Learn.},
month = may,
pages = {235–256},
numpages = {22}
}

@inproceedings{PBM-UCB,
author = {Lagr\'{e}e, Paul and Vernade, Claire and Capp\'{e}, Olivier},
title = {Multiple-play bandits in the position-based model},
year = {2016},
isbn = {9781510838819},
publisher = {Curran Associates Inc.},
address = {Red Hook, NY, USA},
booktitle = {Proceedings of the 30th International Conference on Neural Information Processing Systems},
pages = {1605–1613},
numpages = {9},
location = {Barcelona, Spain},
series = {NIPS'16}
}

@article{Harper2016TheMD,
  title={The MovieLens Datasets: History and Context},
  author={F. Maxwell Harper and Joseph A. Konstan and Joseph A.},
  journal={ACM Trans. Interact. Intell. Syst.},
  year={2016},
  volume={5},
  pages={19:1-19:19},
  url={https://api.semanticscholar.org/CorpusID:16619709}
}

@article{Zong2016CascadingBF,
  title={Cascading Bandits for Large-Scale Recommendation Problems},
  author={Shi Zong and Hao Ni and Kenny Sung and Nan Rosemary Ke and Zheng Wen and Branislav Kveton},
  journal={ArXiv},
  year={2016},
  volume={abs/1603.05359},
  url={https://api.semanticscholar.org/CorpusID:2545548}
}

@book{chuklin2015click,
  author    = {Aleksandr Chuklin and Ilya Markov and Maarten de Rijke},
  title     = {Click Models for Web Search},
  year      = {2015},
  publisher = {Springer},
  address   = {Cham, Switzerland},
  series    = {Synthesis Lectures on Information Concepts, Retrieval, and Services},
  volume    = {43},
  isbn      = {978-3-031-01166-5},
  doi       = {10.1007/978-3-031-02294-4},
  url       = {https://link.springer.com/book/10.1007/978-3-031-02294-4}
}

@inproceedings{10.5555/3666122.3667927,
author = {Zuo, Jinhang and Zhang, Zhiyao and Wang, Zhiyong and Li, Shuai and Hajiesmaili, Mohammad and Wierman, Adam},
title = {Adversarial attacks on online learning to rank with click feedback},
year = {2023},
publisher = {Curran Associates Inc.},
address = {Red Hook, NY, USA},
booktitle = {Proceedings of the 37th International Conference on Neural Information Processing Systems},
articleno = {1805},
numpages = {18},
location = {New Orleans, LA, USA},
series = {NIPS '23}
}

@article{
wang2024adversarial,
title={Adversarial Attacks on Online Learning to Rank with Stochastic Click Models},
author={Zichen Wang and Rishab Balasubramanian and Hui Yuan and chenyu song and Mengdi Wang and Huazheng Wang},
journal={Transactions on Machine Learning Research},
issn={2835-8856},
year={2024},
url={https://openreview.net/forum?id=BKwGowR0Bt},
note={}
}

@inproceedings{10.5555/3305890.3306115,
author = {Zoghi, Masrour and Tunys, Tomas and Ghavamzadeh, Mohammad and Kveton, Branislav and Szepesvari, Csaba and Wen, Zheng},
title = {Online learning to rank in stochastic click models},
year = {2017},
publisher = {JMLR.org},
booktitle = {Proceedings of the 34th International Conference on Machine Learning - Volume 70},
pages = {4199–4208},
numpages = {10},
location = {Sydney, NSW, Australia},
series = {ICML'17}
}

@article{Lattimore2018TopRankAP,
  title={TopRank: A practical algorithm for online stochastic ranking},
  author={Tor Lattimore and Branislav Kveton and Shuai Li and Csaba Szepesvari},
  journal={ArXiv},
  year={2018},
  volume={abs/1806.02248},
  url={https://api.semanticscholar.org/CorpusID:46946360}
}

@inproceedings{10.5555/2969239.2969401,
author = {Kveton, Branislav and Wen, Zheng and Ashkan, Azin and Szepesv\'{a}ri, Csaba},
title = {Combinatorial cascading bandits},
year = {2015},
publisher = {MIT Press},
address = {Cambridge, MA, USA},
booktitle = {Proceedings of the 29th International Conference on Neural Information Processing Systems - Volume 1},
pages = {1450–1458},
numpages = {9},
location = {Montreal, Canada},
series = {NIPS'15}
}

@inproceedings{10.1145/2911451.2914798,
author = {Grotov, Artem and de Rijke, Maarten},
title = {Online Learning to Rank for Information Retrieval: SIGIR 2016 Tutorial},
year = {2016},
isbn = {9781450340694},
publisher = {Association for Computing Machinery},
address = {New York, NY, USA},
url = {https://doi.org/10.1145/2911451.2914798},
doi = {10.1145/2911451.2914798},
booktitle = {Proceedings of the 39th International ACM SIGIR Conference on Research and Development in Information Retrieval},
pages = {1215–1218},
numpages = {4},
location = {Pisa, Italy},
series = {SIGIR '16}
}

@inproceedings{Jun2018AdversarialAO,
  title={Adversarial Attacks on Stochastic Bandits},
  author={Kwang-Sung Jun and Lihong Li and Yuzhe Ma and Xiaojin Zhu},
  booktitle={Neural Information Processing Systems},
  year={2018},
  url={https://api.semanticscholar.org/CorpusID:53104043}
}

@article{Liu2019DataPA,
  title={Data Poisoning Attacks on Stochastic Bandits},
  author={Fang Liu and Ness B. Shroff},
  journal={ArXiv},
  year={2019},
  volume={abs/1905.06494},
  url={https://api.semanticscholar.org/CorpusID:155100228}
}

@inproceedings{Wang2024StealthyAA,
  title={Stealthy Adversarial Attacks on Stochastic Multi-Armed Bandits},
  author={Zhiwei Wang and Huazheng Wang and Hongning Wang},
  booktitle={AAAI Conference on Artificial Intelligence},
  year={2024},
  url={https://api.semanticscholar.org/CorpusID:267770319}
}

@inproceedings{Rangi2021SavingSB,
  title={Saving Stochastic Bandits from Poisoning Attacks via Limited Data Verification},
  author={Anshuka Rangi and Long Tran-Thanh and Haifeng Xu and Massimo Franceschetti},
  booktitle={AAAI Conference on Artificial Intelligence},
  year={2021},
  url={https://api.semanticscholar.org/CorpusID:231925173}
}

@article{liu2009learning,
  title={Learning to rank for information retrieval},
  author={Liu, Tie-Yan and others},
  journal={Foundations and Trends{\textregistered} in Information Retrieval},
  volume={3},
  number={3},
  pages={225--331},
  year={2009},
  publisher={Now Publishers, Inc.}
}

@article{10.5555/3546258.3546476,
author = {Zhong, Zixin and Chueng, Wang Chi and Tan, Vincent Y. F.},
title = {Thompson sampling algorithms for cascading bandits},
year = {2021},
issue_date = {January 2021},
publisher = {JMLR.org},
volume = {22},
number = {1},
issn = {1532-4435},
journal = {J. Mach. Learn. Res.},
month = jan,
articleno = {218},
numpages = {66}
}

@book{lattimore2020bandit,
  title={Bandit algorithms},
  author={Lattimore, Tor and Szepesv{\'a}ri, Csaba},
  year={2020},
  publisher={Cambridge University Press}
}
\bibliographystyle{IEEEtran}

\appendix

\vspace*{-0.5 em}
\section{Algorithms for OLTR}
\vspace*{-0.5 em}
\label{App:Algos}
In this section, we discuss two prominent OLTR algorithms: $\cascadeucb$ (Algorithm \ref{alg:cascadeucb1}) for the Cascade model and  $\pbmucb$ (Algorithm \ref{alg:pbmucb}) for PBM.
\vspace*{-0.5 em}
\subsection{Definitions}
\vspace*{-0.5 em}
 Before presenting the algorithms in detail, we define the following terms:  
\begin{enumerate}  
\item[--] \textit{Number of Recommendations}: \(\reco_a\left(t\right) =\sum_{\tau=1}^{t} \mathds{1}{\{a \in \reclist_\tau\}}\) represents the number of times an OLTR algorithm has recommended an item \( a \) up to round \( t \).  

\item[--] \textit{Number of Examinations}: \(\pulls_a\left(t\right) =\sum_{\tau=1}^{t} \sum_{i=1}^{\nreco} \mathds{1}{\{a_{i,\tau} = a\}} \cdot \exam_{i,\tau}\) denotes the number of times item \( a \) has been examined by users up to time \( t \). 

\item[--] \textit{Number of Clicks}: \(\nclk_a\left(t\right) =\sum_{\tau=1}^{t}  \rew_{a,\tau}\) denotes the number of times item \( a \) has been clicked by users up to time \( t \).

\item[--] \textit{Empirical Mean}: The empirical mean reward for an item is given by \(\hat{\wts}_{a,t}= \sum_{\tau=1}^{t} \)\( \rew_{a,\tau}/{\pulls_a\left(t\right)}\),and is defined as the average number of clicks per examination.  

\item[--] \textit{UCB Index}: The term \( U_a(t) \) represents the upper confidence bound of the attractiveness of item \( a \), as maintained by the OLTR algorithm.  
\end{enumerate}  
\vspace*{-0.5 em}
\begin{algorithm}[H]
\caption{$\cascadeucb$}
\label{alg:cascadeucb1}
\begin{algorithmic}
\State {\bfseries Input:} item set $\totset$, number of recommended items $\nreco$, horizon $T$, exploration parameter $\alpha$
\State {\bfseries Initialize:} $t \leftarrow 1$, $\hat{\wts}_{a,t} \leftarrow 0$, $\pulls_a\left(t\right) \leftarrow 1$, for each $a \in \totset$ 
\While {$t \leq T$}
\State Compute $U_a\left(t\right)$ according to \eqref{eq:csub}
\State Recommend $\reclist_t  =  (a_{1,t}, \dots, a_{\nreco,t})$ to the user, where $a_{1,t} \dots,a_{\nreco,t}$ $\in\totset$ are the $\nreco$ items with largest $U_a\left(t-1\right)$
\State  Observe the user feedback - $\exam_t$ and $\click_t$ (equivalently $\rew_t$)
\For {$a \in \reclist_t$}
\If {$a$ is examined}
\State $\pulls_a\left(t\right) \leftarrow \pulls_{a}\left(t-1\right) +1$
\State$\hat{\wts}_{a,t} \leftarrow \{\hat{\wts}_{a,t-1}\pulls_{a}\left(t-1\right) + \rew_{a,t}\}/{\pulls_a\left(t\right)}$
\Else
\State $\pulls_a\left(t\right) \leftarrow \pulls_{a}\left(t-1\right)$ ; $\hat{\wts}_{a,t} \leftarrow \hat{\wts}_{a,t-1} $
\EndIf
\EndFor
\State $t \leftarrow t + 1$
\EndWhile
\end{algorithmic}
\end{algorithm}
\vspace*{-0.5 em}
\subsection{UCB-based OLTR Algorithms}
\vspace*{-0.5 em}
In each round, both CascadeUCB1 and PBM-UCB compute the UCB index for each item based on the corresponding number of clicks, recommendations, and examinations, then select the \( K \) items with the highest UCB indices and recommend them as \(\reclist_t\), ranking them in descending order of their UCB values. The computation of UCB differs significantly between the two algorithms; the computation in $\cascadeucb$ is similar to that of the classical UCB1 algorithm \cite{10.1023/A:1013689704352}, with a parameter $\alpha>1$,
\begin{equation}
    U_a\left(t\right) = \hat{\wts}_{a,t} + \sqrt{\frac{\alpha\log t}{\pulls_a\left(t\right)}},
    \label{eq:csub}
\end{equation}
whereas $\pbmucb$, where the learning agent cannot observe the examination feedback $\exam_t$, uses
\begin{equation}
    U_a\left(t\right) = \frac{\nclk_a(t)}{\Tilde{\pulls}_a\left(t\right)} + \sqrt{\frac{\alpha \reco_a\left(t\right) \log t}{{\Tilde{\pulls}}_a^2\left(t\right)}},
    \label{eq:pbub}
\end{equation}
where $\Tilde{\pulls}_a\left(t\right)$ is an unbiased estimator for the number of examinations, given by $\Tilde{\pulls}_a\left(t\right) = \sum^{T}_{\tau =1}\sum^{\nreco}_{i =1} p_i \mathds{1}{\left\{a=a_{i,\tau}\right\}}  $.

 The $\cascadeucb$ algorithm was originally introduced in \cite{kveton15cascade},
which provided an $O(\log T)$ upper bound on its expected cumulative regret for a given horizon $T$. The corresponding $\pbmucb$ algorithm for position-based feedback was proposed by \cite{PBM-UCB} with a similar $O(\log T)$ regret bound. 
\vspace*{-0.5 em}
\begin{algorithm}[H]
\caption{$\pbmucb$}
\label{alg:pbmucb}
\begin{algorithmic}
\State {\bfseries Input:} item set $\totset$, number of recommended items $\nreco$, position bias $\poslist$, horizon $T$, exploration parameter $\alpha$
\State {\bfseries Initialize:} $t \leftarrow 1$, $\nclk_a\left(t\right) \leftarrow 0$, $\reco_a\left(t\right) \leftarrow 1$, for each $a \in \totset$ 
\While {$t \leq T$}
\State Compute $U_a\left(t\right)$ according to  \eqref{eq:pbub}
\State Recommend $\reclist_t  =  (a_{1,t}, \dots, a_{\nreco,t})$ to the user, where $a_{1,t} \dots,a_{\nreco,t}$ $\in\totset$ are the $\nreco$ items with largest $U_a\left(t-1\right)$

\State  Observe the user feedback  $\click_t$ and $\rew_t$
\For {$a \in \reclist_t$}
\State $\reco_a\left(t\right) \leftarrow \reco_{a}\left(t-1\right) +1$
\State $\nclk_a\left(t\right) \leftarrow \nclk_{a}\left(t-1\right) +\rew_{a,t}$
\EndFor
\State $t \leftarrow t + 1$
\EndWhile
\end{algorithmic}
\end{algorithm}
\vspace*{-0.5 em}
\begin{remark} 
\label{rem:alpha}
The parameter $\alpha$, present in both the OLTR algorithms (check \eqref{eq:csub} and \eqref{eq:pbub}), is often referred to as the exploration parameter or the confidence parameter. The parameter is commonly used for the UCB-based online learning algorithms to define the confidence radii. Intuitively, a higher $\alpha$ ensures that the OLTR is more confident of its learned list being optimal, but it comes with a higher regret bound for the algorithms due to increased exploration of sub-optimal algorithms. Both the OLTR algorithms and their corresponding attack strategies are valid for any $\alpha>1$.
\end{remark}
\vspace*{-0.5 em}
\section{Key Proofs}
\label{App:Key Proofs}
\vspace*{-0.5 em}
This section provides the proofs for the theorems given in Section \ref{sec:genattack}.
\subsection{Proof of Theorem \ref{thm:cascadeatk}}
\label{app:proof-cas}
Some of the consequences of the $\ofcs$ attack on $\cascadeucb$ are stated below in the following lemmas, which are used to prove Theorem \ref{thm:cascadeatk}. The proofs of these lemmas are discussed in the Appendix \ref{App: Add Proofs}. 
\begin{lemma} 
\label{lemma:c1_cas}
After the first phase of $\ofcs$ with the value of $\ca$ given in Section \ref{subsec:cas-skel},
 $U_{a}(\ca+1) \leq \pmy \ \;\forall\; a \in \totset$. 
 \end{lemma}
From Lemma \ref{lemma:c1_cas}, we conclude that, at the end of phase 1 of $\ofcs$, UCBs for all the items in $\totset$ fall below $\pmy$.
\begin{lemma} 
\label{lemma:c2_cas}
After the second phase of $\ofcs$ with $\ca$ and $\cb$  given in Section \ref{subsec:cas-skel}, $U_{a}(\ca+\cb+1)  > \pmy$ for all $ a \in \tarlist$ and  $U_{a'}(\ca+\cb+1)  \leq \pmy$ for all $ a' \notin \tarlist$.
\end{lemma}
 Thus, at the end of Phase 2 of $\ofcs$, the UCBs of all items in the target set $\tarset$ exceed $\pmy$, while those for the remaining items in $\totset$ stay below $\pmy$.
\begin{lemma}
\label{lemma:eff_cas} Without any further manipulation in the third phase of $\ofcs$, an item $a\in \tarlist$ maintains $U_a\left(t\right)>\pmy$, with probability
\begin{equation}
\label{cascade_success}
\Pr\{U_a\left(t\right)>\pmy\; \forall \; \ca+\cb<t \leq T\}  \geq 1-1/T,
\end{equation}
\end{lemma}
Thus, with high probability, each item in $\tarlist$ maintains its UCB above $\pmy$ throughout the third phase of $\ofcs$.
This brings us to the Proof of Theorem \ref{thm:cascadeatk} which uses the abovementioned lemmas. 
\begin{proof}[Proof of Theorem \ref{thm:cascadeatk}]
 According to Lemma \ref{lemma:c1_cas} and Lemma \ref{lemma:c2_cas}, the $\ofcs$ strategy outlined in Algorithm \ref{alg:observation_free_cascade} with the values of $\ca$ and $\cb$ given by  Section \ref{subsec:cas-skel} ensures that at the end of first \[\cost = \ca + \cb\] rounds, $U_{a}(\cost+1)  \geq \pmy \; \forall\; a \in \tarlist$ and $U_{a'}(\cost+1)  \leq \pmy \; \forall\; a' \notin \tarlist$. Additionally, Lemma \ref{lemma:eff_cas},  guarantees that any  \( a \in \tarlist \) maintains an UCB above $\pmy$ with a probability \( \geq 1 - 1/ T \). 
 
 Let $\Lambda_i$ be the event that $U_{\tarlist[i]}\left(t\right)>\pmy\;$ for all t in $\{C+1,\dots,T\}$. Extending the analysis simultaneously to all the items in $\tarlist$, for $t>\cost$, we obtain
\[
\Pr\{\Lambda_1 \cap \Lambda_2 \dots \cap 
\Lambda_\nreco \}  \geq 1 - \nreco / T.
\]

With $ U_{a'}\left(\cost+1\right) \leq \pmy$ for all $a' \notin \tarlist$ and $U_a(t) > \pmy$ for all $a \in \tarlist$ in the third phase, none of the non-target items would receive any examinations after $t=\cost$. Consequently, with  high probability, \( U_{a'}(t) = U_{a'}(\cost+1) \leq \pmy \) for all \( a' \notin \tarlist \) and \( \cost < t \leq T \). This ensures that \( \reclist_t  \) is always a permutation of \(\tarlist\) during the final \( T - \cost\) rounds and all items in  \( \tarset \) are consistently recommended by $\cascadeucb$ in the third phase. 

The probability of the \textit{success event} \( \sigma \), i.e., $\ofcs$ successfully misleading $\cascadeucb$ into recommending items from the target set $\tarset$ as part of $\reclist_t$ for $T - o(T)$ rounds is given by \(
    \Pr(\sigma) \geq 1 - \nreco{/T}.\)

According to Theorem 2 in \cite{kveton15cascade}, in the absence of any external manipulations, $\cascadeucb$ incurs an expected regret of $O(\log T)$. In contrast, the cumulative $T$-step  regret of $\cascadeucb$ with $\ofcs$ is lower-bounded by
\begin{align*}
\regret(T) \geq \sum_{t=\cost+1}^T \mathds{E}\left[\reglist(\reclist_t, \wts)\right] \geq\Pr(\sigma)\cdot\sum_{t=\cost+1}^T\min_{\pmst(\tarlist)}\reglist(\reclist, \wts) )
\geq \left(1-\frac{\nreco}{T}\right)\cdot(T-\cost)\cdot\reglist(\tarlist, \wts).
\end{align*}
$\regret(T)=\Omega(T)$ if $\reglist(\tarlist, \wts)> 0$, which holds true for any $\tarlist$ containing at least one element not in $\reclist^\ast$. Thus, $\ofcs$  is not only successful in recommending the target items for $T-o(T)$ rounds but also ensures a linear regret for $\cascadeucb$.
\end{proof}

\subsection{Proof of Theorem \ref{thm:pbmatk}}
\label{app:proof-pbm}
Some of the major results for $\ofpb$ are stated below in the following lemmas, which are used to prove Theorem \ref{thm:cascadeatk}. The proofs of these lemmas are discussed in the Appendix \ref{App: Add Proofs}. 
 \begin{lemma}
    \label{lemma:c1_pbm}
    After the first phase of $\ofpb$ with the value of $\ca$ given in Section \ref{subsec:pbm-skel},
 the probability $\Pr\{\succa \} =1$, where $\succa = \mathcal{E}_1 \cap \mathcal{E}_2 \dots \cap \mathcal{E}_\ntot $ and $\mathcal{E}_a$ is the event that $U_{a}(\ca+1)\leq \pmy$.
 \end{lemma}
  At the end of the first phase, UCBs for all the items in $\totset$ fall below $\pmy$ with a probability of one.
 \begin{lemma} 
\label{lemma:c2_pbm}
After the second phase of $\ofpb$ with the value of $\ca$ and $\cb$ given in Section \ref{subsec:pbm-skel}, $\Pr\{\succb \} \geq 1-\nreco/T$, where $\succb=\mathcal{E'}_1 \cap \mathcal{E'}_2 \dots \cap \mathcal{E'}_\nreco$ and $\mathcal{E'}_i$ is the event that $U_{\tarlist[i]}\left(\ca+\cb +1\right)>\pmy$. Additionally, $U_{a'}(\ca+\cb+1)  \leq \pmy$ for all $ a' \notin \tarlist$.
\end{lemma} 
At the end of Phase 2 of $\ofpb$, the UCBs of all items in $\tarset$ exceed $\pmy$, while those of the remaining items in $\totset$ stay below $\pmy$ with a high probability.
\begin{lemma}
\label{lemma:eff_pbm} 
Given the success of the first two phases and in the absence of further manipulation in the third phase of $\ofpb$, an item \( a \in \tarlist \) maintains UCB \( U_a(t) > \pmy \) under $\pbmucb$ with a probability
\[
\Pr\{\succc\; |\; \succa \cap \succb  \}  \geq 1-\nreco/T,
\]
where  $\succc=\mathcal{E''}_1 \cap \mathcal{E''}_2 \dots \cap \mathcal{E''}_\nreco$ and $\mathcal{E''}_i$ is the event that $U_{\tarlist[i]}\left(t\right)>\pmy$ for all t in $\{C+1,\dots,T\}$.
\end{lemma}
If the first two phases of $\ofpb$ were successful, then, with high probability, each item in $\totset$ maintains its UCB above $\pmy$ throughout the third phase of $\ofpb$.

This brings us to the Proof of Theorem \ref{thm:pbmatk} which uses the abovementioned lemmas. 
\begin{proof}[Proof of Theorem \ref{thm:pbmatk}] The Lemmas \ref{lemma:c1_pbm} and \ref{lemma:c2_pbm} provide the guarantees of success for the first two phases of $\ofpb$. In the final phase of the attack, Lemma \ref{lemma:eff_pbm} asserts that if the earlier phases were successful, UCB $U_a(t) > \pmy$ for all $a \in \tarlist$, with probability at least $1 - \nreco/T$ while  $\ca + \cb<t\leq T$. 

Using a proof sketch similar to that of Theorem \ref{thm:cascadeatk}, the probability of $\ofpb$ successfully misleading the $\pbmucb$ into recommending items from $\tarset$ as part of $\reclist_t$ for $T - o(T)$ rounds is given by
\begin{align*}
\Pr(\mathrm{\sigma}) &= 1 - \Pr(\mathrm{failure})
\geq 1 - \Pr(\widebar{\succa}) - \Pr(\widebar{\succb}) - \Pr(\widebar{\succc} | \succa \cap \succb  ) 
  \geq 1 - \frac{2 \nreco}{T}.
\end{align*}
Thus, $\ofpb$ efficiently misleads $\pbmucb$ into recommending items from the target set $\tarset$ for $T-o(T)$ rounds with high probability with a manipulation of the order $O( \log T)$. 

Theorem 9 in \cite{PBM-UCB} states that in the absence of any attack, the $\pbmucb$ algorithm incurs an $O(\log T)$ regret. We define \( \displaystyle
    \tarlist' = \argmin_{\pmst(\tarlist)}(\reglist(\reclist, \wts)).\)

The cumulative $T$-step regret for $\pbmucb$ under $\ofpb$ is lower-bounded by
\begin{align*}
\regret(T) \geq \sum_{t=\cost+1}^T \mathds{E}\left[\reglist(\reclist_t, \wts)\right] \geq\Pr(\sigma)\cdot\sum_{t=\cost+1}^T\min_{\pmst(\tarlist)}\reglist(\reclist, \wts) )
\geq \left(1-\frac{\nreco}{T}\right)\cdot(T-\cost)\cdot\reglist(\tarlist', \wts).
\end{align*}
Therefore, $\ofpb$, along with successfully recommending the target items for $T-o(T)$ rounds, also forces an $\Omega(T)$ linear regret on the $\pbmucb$  algorithm.
\end{proof}

\section{Additional Proofs}
\label{App: Add Proofs}
\subsection{Proof of Lemma \ref{lemma:c1_cas} }
\label{subsec: ca_cascade}

We use the following lemma to prove Lemma \ref{lemma:c1_cas}.
\begin{lemma} 
\label{lemma:round-robin}
Under $\cascadeucb$, given a set $\Gamma$ of items with equal empirical means and UCBs at the start of a round $t$. Any item $a \in \Gamma$ that is examined but not clicked in round $t$ cannot be examined again until the remaining items in $\Gamma$ get an examination. The items in $\Gamma$ are recommended and examined in a round-robin fashion.
\end{lemma}
\begin{proof}[Proof of Lemma  \ref{lemma:round-robin}]
    Let there be two items $a,a' \in \Gamma$ with $\hat{\wts}_{a,t}=\hat{\wts}_{a',t}$ and $U_{a} \left(t\right)=U_{a} \left(t\right)$, such that $a$ is examined but not clicked at round $t$, while $a'$ is not examined in the same round. Let $a$ get another examination in round $t_1>t$, while $a'$ has no examinations in rounds $\{t,\dots,t_1\}$. Under $\cascadeucb$, if item $a$ gets an examination before $a'$ in round $t_1$, it would imply that $U_{a} \left(t_1\right)\geq U_{a'} \left(t_1\right)$, i.e. $U_{a} \left(t_1\right)\geq U_{a} \left(t\right)$, which is contradictory to the fact that $\hat{\wts}_{a,t_1}<\hat{\wts}_{a,t}$ and $\pulls_a(t) > \pulls_a(t_1)$, as item $a$ had been examined but not clicked at round $t$. Thus, an item $a \in \Gamma$ that is examined but not clicked in round $t$ cannot be examined again until all the remaining items $a' \in \Gamma/\{a\}$ get an examination. This implies that the items in $\Gamma$ will be recommended and examined in a round-robin manner under $\cascadeucb$.
\end{proof}
This brings us to the Proof of Lemma \ref{lemma:c1_cas} which uses the abovementioned lemmas. 
\begin{proof}[Proof of Lemma \ref{lemma:c1_cas}]
While $t\leq\ca$, the rewards for all the items are set to 0. Therefore, in each round, the items are examined and ignored (not clicked) in a round-robin pattern in batches of $\nreco$ items according to Lemma \ref{lemma:round-robin}. After the completion of the first phase, the number of pulls for each item $a$ is given by $\pulls_{a}\left(\ca+1\right)=\frac{\nreco \ca }{L}$, and therefore with the value of $\ca$ given in Section \ref{subsec:casattackucb},  
\[U_{a} \left(\ca+1\right) \leq \sqrt{\frac{\ntot\alpha \log T}{\nreco \ca }} \leq \pmy \; \forall\; a \in \totset. \]
\end{proof}

\subsection{Proof of Lemma \ref{lemma:c2_cas}}
\label{subsec: cb_cascade}
\begin{proof}[Proof of Lemma \ref{lemma:c2_cas}]
 The second phase is split into $\nreco$  sub-phases with the $i^{\mathrm{th}}$ sub-phase lasting from round $ \ca +\frac{(i-1)\cb}{\nreco}+1$ to round $\ca +\frac{i\cb}{\nreco}$ where $1\leq i\leq \nreco$. During the $i^{th}$ sub-phase, the reward for $\tarlist[i]$ is fixed as 1, while the rewards for the rest of the items are set to 0. 

Assume that at the start of the \(i^{\text{th}}\) sub-phase, \(\tarlist[i]\notin \reclist_t\). At this point, the top \(i-1\) positions of \(\reclist_t\) are occupied by a permutation of the first \(i-1\) members of \(\tarlist\). Until \(\tarlist[i] \in \reclist_t\), none of the items in \(\reclist_t\) will be clicked. As a result, the lower \(\nreco-i+1\) positions in \(\reclist_t\) are filled in a round-robin manner (as seen in Lemma \ref{lemma:round-robin}) from the remaining \(\ntot-i+1\) items in \(\totset\) with zero empirical mean. Once $\tarlist[i] \in \reclist_t$, the item $\tarlist[i]$ is clicked for the rest of the sub-phase. The number of rounds required to search for $\tarlist[i]$ at the start of the $i^{\mathrm{th}}$ sub-phase is upper-bounded by \[\left\lceil\frac{\ntot-i+1}{\nreco-i+1}\right\rceil \leq \ntot -\nreco+1,\hspace{4mm}1\leq i\leq \nreco.\] With each sub-phase lasting for an equal number of rounds, the empirical mean for all the items in $\tarlist$ (and thus $\tarset$) being greater than $\pmy$ at the end of round two is ensured by
\[\frac{\cb/\nreco - \left(\ntot-\nreco+1\right)}{\nreco\ca/\ntot+\cb} \geq \pmy.\]
Thus, with $\cb$ equal to the value given in Section \ref{subsec:cas-skel},  $\hat{\wts}_{a,\cost+1}\geq \pmy$ and $U_{a}(\cost+1)  \geq \pmy$ for all  $a \in \tarlist$. Additionally, the items not in $\tarlist$ continue to be ignored (not clicked) during each of the recommendations, thus pulling down their UCBs further below $\pmy$ by the end of the second phase, i.e. $U_{a'}(\cost+1)  \leq \pmy$ for all  $a' \notin \tarlist$ . These two events occur simultaneously in the second phase of $\ofcs$ proving Lemma \ref{lemma:c2_cas}. Note that $\cb$ is a positive quantity as $\pmy < 1/\nreco$.
\end{proof}

\noindent\textbf{Note.} The decision to divide the second phase of $\ofcs$ into $\nreco$ equal-length sub-phases is a design choice and alternative bounds for $\cb$ can be derived by using sub-phases of varying lengths. A possible suggestion can be to give $\left\lceil \frac{(\nreco-i)\cb}{\nreco(\nreco+1)}\right\rceil$ rounds to the $i^{\mathrm{th}}$ sub-phase. Intuitively, giving more rounds to the earlier sub-phases is beneficial, as their corresponding target items get observed and ignored for all the later sub-phases, thus pulling down their UCBs. 
\subsection{Proof of Lemma \ref{lemma:eff_cas} }
\label{subsec: eff_cascade}

\begin{proof}[Proof of Lemma \ref{lemma:eff_cas}]
\label{proof:eff_cas}
For the last $T-\cost$ rounds, all the items in $\tarlist$ should be recommended to the user with a high probability. At the beginning of the third phase, $ U_a\left(\cost+1\right) > \pmy$ for all $a \in \tarlist$.

     Let an item $a \in \tarlist$ exist, which receives $m$ examinations in the initial two phases and further $n$ examinations rest of the rounds. By Hoeffding's inequality, the number of clicks for the item $a$ in the third phase until some time $t\leq T$ will be greater than $n\wts_a-\sqrt{n\log t}$ with probability at least $1-1/T$ (Check Section A.2 in \cite{NEURIPS2021_be315e7f}), and thus, the UCB for this item at time $\cost<t\leq T$ is given by
\begin{align*}
U_a\left(t\right) &= \hat{\wts}_a\left(t\right) + \sqrt{\frac{\alpha \log t}{n+m}} 
\geq \frac{m\pmy +n\wts_a}{n+m} + \sqrt{\frac{\alpha \log t}{n+m}} - \frac{\sqrt{n\log t}}{n+m} > \pmy,
\end{align*}

which follows from the fact that $\wts_a>\pmy$ and $\alpha>1$.Therefore, the UCB of any item in  $\tarlist$ will be greater than $\pmy$ for any round while $\cost<t\leq T$ with the probability given in  \eqref{cascade_success}.
\end{proof}

\subsection{Proof of Lemma \ref{lemma:c1_pbm} }
\begin{proof}[Proof of Lemma \ref{lemma:c1_pbm}]

We use the following lemma to prove Lemma \ref{lemma:c1_pbm}.
\begin{lemma} 
\label{lemma:round-robin-pbm}
Given two items $a_i,a_j\in\totset$, $\reco_{a_j}(t) \leq \lambda^2_{p} \reco_{a_i}(t)+1$ for any $t\leq\ca$ during the first phase of $\ofpb$.
\end{lemma}
\begin{proof}[Proof of Lemma \ref{lemma:round-robin-pbm}]
We prove the lemma using inductive reasoning. Consider any two arbitrarily chosen items \(a_i, a_j \in \totset\). 

\textbf{Base Case} (\(t = 1\)):  
   At the time of initialization (\(t \leftarrow 1\)), we have \(\reco_{a_i}(t) = \reco_{a_j}(t) = 1\).  
   Thus, the lemma statement holds for \(t = 1\) since \(\reco_{a_j}(t) \leq \lambda_p^2 \reco_{a_i}(t) + 1\).

\textbf{Inductive Step} (\(t = t_o + 1\)):  
   Assume that at the beginning of some round \(t_o < \ca\), the inductive hypothesis holds:  
   \[
   \reco_{a_j}(t_o) \leq \lambda_p^2 \reco_{a_i}(t_o) + 1.
   \]
   We now show that the lemma statement holds for \(t = t_o + 1\).  

   At \(t = t_o\), one of the following mutually exclusive and exhaustive events must occur-  
\begin{enumerate}[left=2pt]
      \item \(a_i \in \reclist_{t_o}\) and \(a_j \in \reclist_{t_o}\): 
     In this case, at the start of round \(t_o + 1\),  
     \[
     \reco_{a_i}(t_o + 1) = \reco_{a_i}(t_o) + 1, \quad \reco_{a_j}(t_o + 1) = \reco_{a_j}(t_o) + 1.
     \] 
     Thus,  
      \(
     \reco_{a_j}(t_o + 1) \leq \lambda_p^2 \reco_{a_i}(t_o) + 2 \leq \lambda_p^2 \reco_{a_i}(t_o + 1) + 1.
     \)  
     The lemma holds in this case.

   \item \({a_i} \in \reclist_{t_o}\) and \(a_j \notin \reclist_{t_o}\):  
     In this case,  
      \[
     \reco_{a_i}(t_o + 1) = \reco_{a_i}(t_o) + 1, \quad \reco_{a_j}(t_o + 1) = \reco_{a_j}(t_o).
     \]
     Therefore,  
      \(
     \reco_{a_j}(t_o + 1) \leq \lambda_p^2 \reco_{a_i}(t_o) + 1 \leq \lambda_p^2 \reco_a(t_o + 1) + 1.
     \) 
     
     The lemma holds in this case as well.

  \item \({a_i} \notin \reclist_{t_o}\) and \(a_j \notin \reclist_{t_o}\):  
     In this scenario,  
      \[
     \reco_{a_i}(t_o + 1) = \reco_{a_i}(t_o), \quad \reco_{a_j}(t_o + 1) = \reco_{a_j}(t_o).
     \]
     Hence, the lemma statement naturally holds since the quantities remain unchanged.

   \item \({a_i} \notin \reclist_{t_o}\) and \(a_j \in \reclist_{t_o}\): This occurs if and only if \(U_a(t_o) \leq U_{a_j}(t_o)\). For \(t_o < \ca\) in \(\ofpb\), we know:  
     \[
     \sqrt{\frac{\alpha \log t}{p_1^2 \reco_{a_i}(t_o)}} \leq U_{a_i}(t_o) \leq \sqrt{\frac{\alpha \log t}{p_\nreco^2 \reco_a(t_o)}},
     \] 
     \[
     \sqrt{\frac{\alpha \log t}{p_\nreco^2 \reco_{a_j}(t_o)}} \leq U_{a_j}(t_o) \leq \sqrt{\frac{\alpha \log t}{p_\nreco^2 \reco_{a_j}(t_o)}}.
    \] 
     This is a consequence of the fact that \[p_\nreco \reco_{a}(t) \leq \Tilde{\pulls}_{a}(t) \leq p_1 \reco_{a}(t)\]for all $a\in \totset$ (follows directly from the definition of $\Tilde{\pulls}_{a}(t)$). Therefore, \(U_{a_i}(t_o) \leq U_{a_j}(t_o)\) is possible iff \[
     \reco_{a_j}(t_o) \leq \lambda_p^2 \reco_{a_i}(t_o).
     \]
     Thus,  
      \(
     \reco_{a_j}(t_o + 1) = \reco_{a_j}(t_o) + 1 \leq \lambda_p^2 \reco_a(t_o + 1) + 1.
    \)The lemma holds in this case as well.
\end{enumerate}
The lemma statement is always true for $t=t_o +1$. Therefore, by induction, the lemma is true for all rounds \(t \leq \ca\).
\end{proof}

During the first phase of $\ofpb$, the attack ensures that at $t=\ca+1$, $U_a(t) \leq \sqrt{\frac{\alpha \log T}{p_\nreco^2 \reco_a(t)}}\leq \pmy$ for all $a\in \totset$ and $t\leq\ca$. For this to happen, we need
\[\reco_a(\ca+1)\geq \frac{\alpha \log T}{\psq}\geq \frac{\alpha \log t}{\psq} \; \forall\; a \in \totset.\]
Let \(m\) denote the minimum number of recommendations received by an item in \(\totset\) during the first phase of \(\ofpb\). The corresponding maximum number of recommendations will be less than \(\lambda_p^2 m + 1\), based on the Lemma \ref{lemma:round-robin-pbm}. 

The total number of recommendations during the first phase is equal to \(\nreco\ca\), which implies 
\begin{equation}
\label{eq:reco-sum}
\sum_{a=1}^\ntot \reco_a(\ca) = \nreco\ca.
\end{equation}

Since each item receives at least \(m\) and at most \(\lambda_p^2 m + 1\) recommendations, we can write  
\[
\ntot \cdot m \leq \sum_{a=1}^\ntot \reco_a(\ca) \leq \ntot \cdot (\lambda_p^2 m + 1).
\]  

Substituting \(\sum_{a=1}^\ntot \reco_a(\ca) = \nreco\ca\), we obtain 
\begin{equation}
\label{eq:reco-frac}    
\ntot \cdot (\lambda_p^2 m + 1) \geq \nreco\ca.
\end{equation}

On substituting the value of $\ca$ from Section \ref{subsec:pbm-skel}, we get $m\geq \alpha\log T/\psq$ and therefore $U_{a}(\ca+1)\leq \pmy$ for all $a \in \totset$, thus proving Lemma \ref{lemma:c1_pbm}
\end{proof}
\subsection{Proof of Lemma \ref{lemma:c2_pbm} }
\begin{proof}[Proof of Lemma \ref{lemma:c2_pbm}]
Assume an item $a \in \tarlist$ receives $m_a$ recommendations in the first phase and $n_a$ recommendations in the second phase of $\ofpb$. Then at $t=\ca+\cb+1$,
\[U_a(t)> \frac{\nclk_a(t)}{\Tilde{\pulls}_{a}(t)}\geq \frac{\nclk_a(t)}{p_1\reco_{a}(t)} = \frac{\nclk_a(t)}{p_1(n_a + m_a)}.  \]
During this phase, each examination for an item in $\tarlist$ corresponds to a click.

With a probability $\geq1-1/T$, $\nclk_a(t)\geq p_\nreco n_a -\sqrt{n_a \log T}$ for all $t \in \{\ca+1,\dots,\ca+\cb\}$, this follows from an argument similar to the one given in the Proof of \ref{lemma:eff_cas} (Appendix \ref{proof:eff_cas}). During the second phase of $\ofpb$, the attack ensures that
\begin{equation}
    \label{eq:st_pbm}
\frac{ p_\nreco n_a -\sqrt{n_a \log T}}{p_1(n_a + m_a)}\geq \pmy\; \forall\; a \in \totset.
\end{equation}
Carrying forward the arguments given in the proof of Lemma \ref{lemma:c1_pbm}, $m_a$ is upper-bounded by \[m_a \leq \nreco\ca - \frac{(L-1)\alpha\log T}{\psq p^2_\nreco} \; \forall\; a \in \totset. \]
On setting $m_a$ to the maximum possible value and solving \eqref{eq:st_pbm}, we obtain $n_a\geq\gamma \log T$. Looking back at \eqref{eq:reco-sum} and \eqref{eq:reco-frac}, we find that they hold for any $t\leq\ca$ as well. During phase 1, the minimum number of recommendations for any item between time $t$ and $t+\Delta_t$ is lower-bounded by
\[
    m_{\Delta_t}= \frac{1}{\lambda_p^2}\left\{\frac{\nreco \Delta_t}{\ntot}-1\right\},\; t+\Delta_t\leq\ca.
\]
At the beginning of the second phase, the items in \(\tarlist\) are indistinguishable from the rest and receive at least an \(m_{\Delta_t}/\Delta_t\) fraction of recommendations between $t=\ca$ and $t=\ca+\Delta_t$ for \(\Delta_t << \cb\). As the second phase progresses, the empirical means of the items in \(\tarlist\) increase, making them more distinguishable from the rest of the items in \(\totset\). Consequently, the instantaneous fraction of recommendations for these items increases from \(m_{\Delta_t}/\Delta_t\), converging asymptotically to \(1\). Therefore, $n_a\geq m_{\cb}$; on solving
\[
    m_{\cb}= \frac{1}{\lambda_p^2}\left\{\frac{\nreco \cb}{\ntot}-1\right\} = \gamma \log T,
\]
we obtain the value of $\cb$ given in Section \ref{subsec:pbm-skel}. Thus, on setting 
\(
    \cb=\left\lceil\frac{L \left(\lambda^2_p\gamma\log T + 1\right)}{\nreco}  \right\rceil,
\)
with a probability $\geq 1-1/T$, $U_a(t)> \nclk_a(t)/\Tilde{\pulls}_a(t)\geq \pmy$ for a given item $a$ in $\totset$ and $t=\ca+\cb+1$. Extending this analysis to all the items in $\totset$, we obtain $\Pr\{\succb \} \geq 1-\nreco/T$, where $\succb=\mathcal{E'}_1 \cap \mathcal{E'}_2 \dots \cap \mathcal{E'}_\nreco$ and $\mathcal{E'}_i$ is the event that $U_{\tarlist[i]}\left(\ca+\cb +1\right)>\pmy$. Additionally, $U_{a'}(\ca+\cb+1)  \leq \pmy$ for all $ a' \notin \tarlist$ as none of the other items outside of $\tarlist$ receive any clicks during this phase; hence proving Lemma \ref{lemma:c2_pbm}.
\end{proof}
\subsection{Proof of Lemma \ref{lemma:eff_pbm} }
\label{subsec: eff_pbm}
\begin{proof}[Proof of Lemma \ref{lemma:eff_pbm}]
This proof proceeds similar to that of Lemma \ref{lemma:eff_cas}. Let an item $a \in \tarlist$ exist, which receives $m$ recommendations in the initial two phases and further $n$ recommendations during the third phase. For this item in $\totset$, we define the following quantities:
\begin{align*} 
p_{avg1}^{(a)}&=\frac{\sum^{\cost}_{\tau =1}\sum^{\nreco}_{i =1} p_i \mathds{1}{\left\{a=a_{i,\tau}\right\}}}{\sum_{\tau=1}^{\cost} \mathds{1}{\{a \in \reclist_\tau\}}},\; p_{avg2}^{(a)}=\frac{\sum^{t}_{\tau =\cost+1}\sum^{\nreco}_{i =1} p_i \mathds{1}{\left\{a=a_{i,\tau}\right\}}}{\sum_{\tau=\cost+1}^{t} \mathds{1}{\{a \in \reclist_\tau\}}}. 
 \end{align*} 
 
By Hoeffding's inequality, the number of clicks for the item $a$ at any time $t\leq T$ in the third phase will be greater than $n\wts_a p_{avg2}^{(a)}-\sqrt{n\log t}$ with probability at least $1-1/T$, and thus, assuming the success of phases 1 and 2, the UCB for this item at time $\cost<t\leq T$ is given by
{\small
\begin{align*}
U_a\left(t\right) &\geq \frac{\nclk_a(t)}{\Tilde{\pulls}_{a}(t)} + \sqrt{\frac{\alpha \reco_a (t)\log t}{\Tilde{\pulls}_{a}^2(t)}} \geq \frac{m\pmy p_{avg1}^{(a)} +n\wts_ap_{avg2}^{(a)}}{np_{avg1}^{(a)}+ mp_{avg2}^{(a)}} + \frac{\sqrt{\alpha (n+m) \log t}- \sqrt{n\log t}}{np_{avg1}^{(a)}+mp_{avg2}^{(a)}} > \pmy,
\end{align*}}
which follows from the fact that $\wts_a> \pmy$ and $\alpha>1$.Therefore, the UCB of any item in  $\tarlist$ will be greater than $\pmy$ for any round while $\cost<t\leq T$ with the probability $\geq 1-1/T$. Extending this analysis to all the target items, we can say that all items in $\tarlist$ maintain UCBs \( U_a(t) > \pmy \) under $\pbmucb$ with a probability
\[
\Pr\{\succc\; |\; \succa \cap \succb  \}  \geq 1-\nreco/T,\]
where  $\succc=\mathcal{E''}_1 \cap \mathcal{E''}_2 \dots \cap \mathcal{E''}_\nreco$ and $\mathcal{E''}_i$ is the event that $U_{\tarlist[i]}\left(t\right)>\pmy$ for all $C<t\leq T$.
\end{proof}  

\vspace*{-0.5 em}
\section{Supplementary Empirical Analysis}
\label{app:emp}
\vspace*{-0.5 em}
This section offers additional details on the empirical analysis discussed in Section \ref{sec:emp}.

We conduct experiments on the real-world MovieLens dataset, processed in the same way as done in \cite{Zong2016CascadingBF}. In all experiments, we set $T = 5\times10^5$, $\nreco = 3$, $\ntot = 10$, $\alpha=1.5$, and average results over 50 runs. For PBM, we use $\poslist = (0.95, 0.90, 0.85)$. $\totset$ is created by arbitrarily choosing 10 movies from the dataset. All the results in Section \ref{sec:emp} are provided for a given $\totset$, with the following attraction probabilities for the constituent items:
\[w=(0.336, 0.204, 0.163, 0.125,
 0.112,
 0.105,
 0.099,
 0.090,
 0.086,
 0.082).\]
 The rewards are sampled in an I.I.D. manner from the dataset. With $\tarlist = (4,7,10)$ for $\ofcs$ and $\tarlist = (8,9,10)$ for $\ofpb$, we obtain $w_{\text{min}} = 0.082$ for both the target sets, and thus, for both manipulation strategies, we use $\wts_m =0.08$. For the given $\totset$ and the abovementioned parameters, the phase durations (manipulations) of the attack strategies are as follows-
 \vspace*{-0.5em}
 \begin{enumerate}
     \item $\ofcs$: $\ca=10260$ and $\cb=1005$.
     \item $\ofpb$: $\ca=11507$ and $\cb=1304$. 
 \end{enumerate}
The required reward manipulation for both the algorithms is well under $3\%$ of the total number
of rounds, showing that even a small group of adversarial users may greatly affect the outcome of the learning algorithm. Note that there is a significant gap between the attraction probability of the items in $\reclist^\ast$ and the rest, with $\approx25\%$ difference in that of items 3 and 4. Therefore, the stated amount of reward manipulation does not significantly affect the overall distribution.
\begin{table}[t]
\centering
\renewcommand{\arraystretch}{1.1}
\vspace*{-1em}
\caption{\small Comparison of costs and regrets for different attack strategies with horizon $T=5\times 10^5$.}

\begin{minipage}{0.48\linewidth}
\centering

\subcaption{\small On CascadeUCB1}

\resizebox{\linewidth}{!}{%
\begin{tabular}{lcc}
\toprule
Attack    & Manipulations & $\regret(T)$ \\ \midrule
No Attack & - & $2.233\times10^3$ \\
CascadeATQ & 11265 & $1.509\times10^4$\\
CascadeAlphaAtk & 2927& $1.457\times10^5$\\
$\ofcs$  & 11265& $1.473\times10^5$\\
\bottomrule
\end{tabular}}
\label{tab:cost_cas}
\end{minipage}
\hfill
\begin{minipage}{0.45\linewidth}
\centering

\subcaption{\small On PBM-UCB}

\resizebox{\linewidth}{!}{%
\begin{tabular}{lcc}
\toprule
Attack    & Manipulations & $\regret(T)$ \\ \midrule
No Attack & - & $2.507\times10^3$ \\
PBMATQ & 12811 & $2.153\times10^4$\\
PBMAlphaAtk & 3432& $1.989\times10^5$\\
$\ofpb$  & 12811& $1.922\times10^5$\\
\bottomrule
\end{tabular}}
\label{tab:cost_pbm}
\end{minipage}
\vspace*{-1em}
\label{tab:cost}
\end{table}

In Section \ref{sec:emp}, we presented the results for both the attack strategies and showcased how they are able to successfully mislead $\cascadeucb$ and $\pbmucb$ to recommend their respective target items for a large majority of the rounds in the third phase.
\begin{figure}[h]
\centering
\begin{subfigure}[h]{0.49\linewidth}
    \centering
    \includegraphics[width=\linewidth]{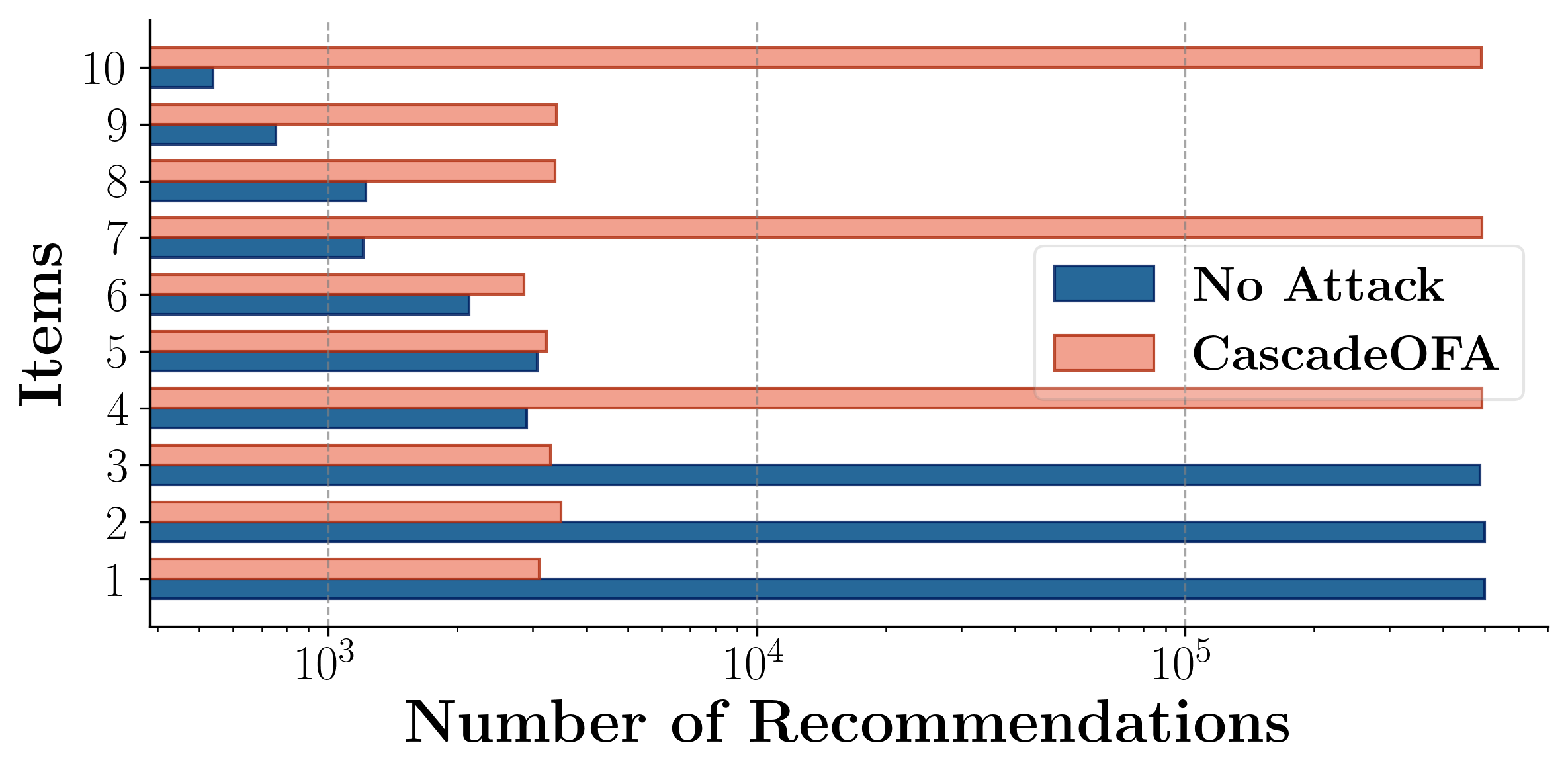}
    \subcaption{CascadeKL-UCB}
    \label{fig:success_kl}
\end{subfigure}
\hfill
\begin{subfigure}[h]{0.49\linewidth}
    \centering
    \includegraphics[width=\linewidth]{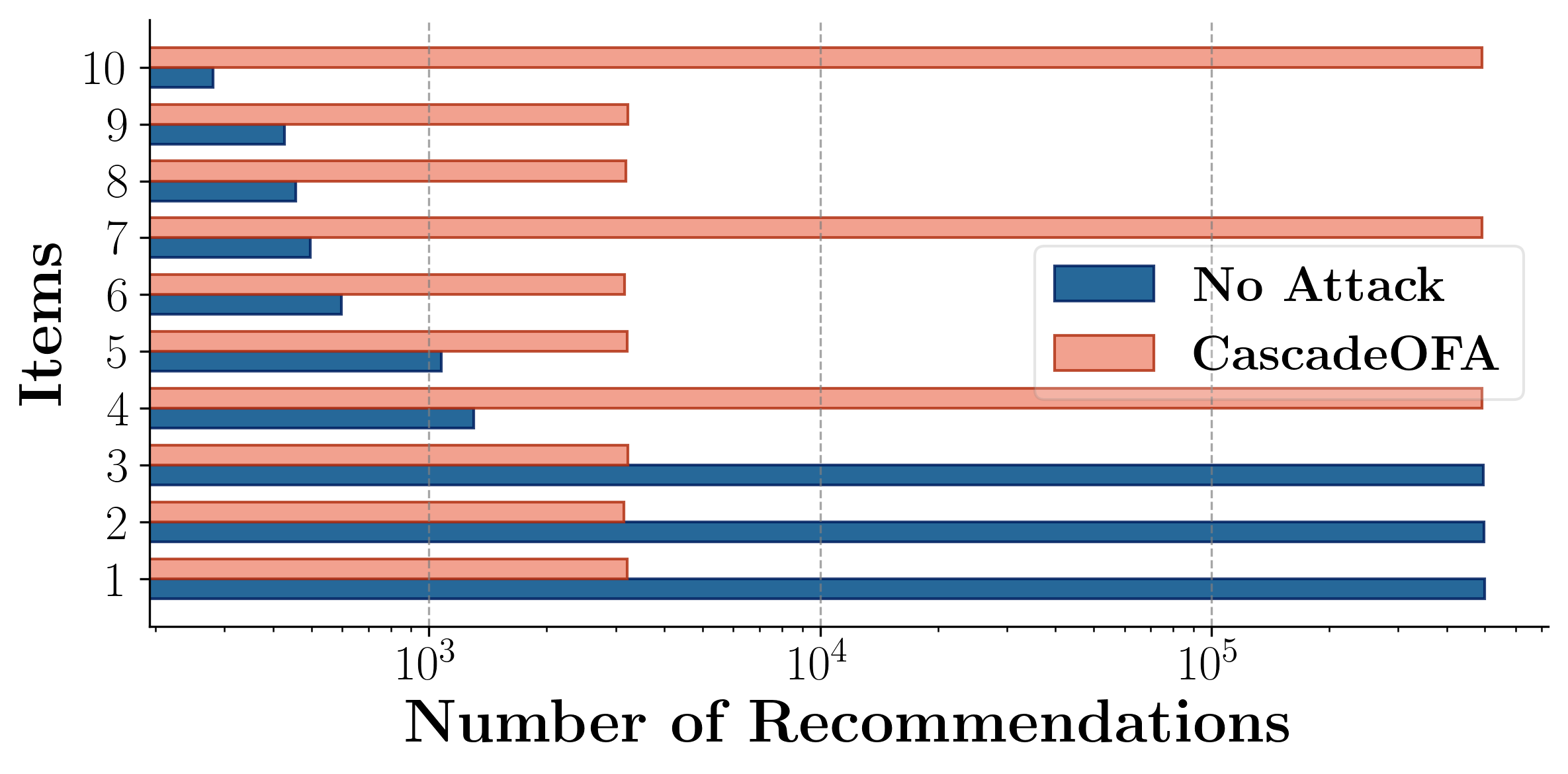}
    \subcaption{TS-Cascade}
    \label{fig:success_ts}
\end{subfigure}
\caption{The number of recommendations for the target items with and without $\ofcs$.}
\vspace*{-0.5em}
\label{fig:success1}
\end{figure}

We benchmarked our strategies against existing attack techniques for OLTR, including the attack-then-quit strategies (CascadeATQ and PBMATQ), adapted from the one given in \cite{wang2024adversarial}, which is effective against arm-elimination based algorithms but fails to mislead UCB-based algorithms. We also compared against the AlphaATk strategies (CascadeAlphaAtk and PBMAlphaAtk) from \cite{10.5555/3666122.3667927}, which require access to the user feedback received by the algorithm in each round. 

CascadeATQ and PBMATQ attacks were configured in a manner to require the same amount of reward manipulation as $\ofcs$ and $\ofpb$, respectively. For both the AlphaAtk strategies, parameters $\Delta_0$ and $\delta$ (defined in \cite{10.5555/3666122.3667927} ) are set to 0.1, following Appendix B.1 in \cite{10.5555/3666122.3667927}.

We present the required reward manipulations and the corresponding regret implications for the aforementioned attack strategies in Tables \ref{tab:cost_cas} and \ref{tab:cost_pbm}. While our strategies and the AlphaAtk strategies significantly increase the regret of OLTR algorithms, the regret induced by the ATQ strategies remains relatively limited compared to the amount of reward manipulation they require. These results are qualitatively consistent with Figure \ref{fig:regret}.

Having shown the efficacy of our attack strategies for UCB-based algorithms, we conducted empirical experiments with some of the other popular OLTR
algorithms as well, using the same parameters as earlier. Figure \ref{fig:success1} presents the results for CascadeKL-UCB and TS-Cascade in the presence of $\ofcs$, with  $\ca=10260$ and $\cb=1005$ (same as earlier). $\ofcs$ is able to mislead both the OLTR algorithms into recommending the target items for a large majority of rounds.

Motivated by the empirical results, we plan to extend our theoretical analysis of $\ofcs$ to other Cascade model-based OLTR algorithms, with the goal of establishing regret bounds and recommendation guarantees for algorithms such as CascadeKL-UCB and TS-Cascade.

\end{document}